\documentclass{article}
\usepackage{graphicx}

\title{The Importance of Being Smoothly Calibrated}

\author{Parikshit Gopalan\\
Apple
\and Konstantinos Stavropoulos\thanks{Work done during an internship at Apple.}\\
UT Austin\and Kunal Talwar\\
Apple
\and Pranay Tankala\samethanks\\
Harvard}
\date{}
\usepackage[utf8]{inputenc}
\usepackage[T1]{fontenc}
\usepackage{amsmath, amssymb, amsthm, bbm, graphicx, url}
\usepackage{amsfonts,fullpage}
\usepackage{float}
\usepackage{graphicx}

\usepackage{bbm}
\usepackage{xcolor}
\usepackage{natbib}
\usepackage[boxruled,linesnumbered,vlined]{algorithm2e}
\usepackage[pagebackref]{hyperref}
\hypersetup{
    colorlinks=true,
    linktocpage=true,
    linkcolor=blue!50!black,  
    citecolor=blue!50!black,  
    urlcolor=blue!50!black    
}
\usepackage[capitalize]{cleveref}
\usepackage{dsfont}
\usepackage{xspace}
\usepackage{comment}
\usepackage{tikz}

\usepackage{thmtools}
\usepackage{thm-restate}

\newcommand{\eat}[1]{}

\newcommand{\samethanks}{\footnotemark[\value{footnote}]}

\newtheorem*{definition*}{Definition}
\newtheorem*{proposition*}{Proposition}
\newtheorem*{corollary*}{Corollary}

\newtheorem{theorem}{Theorem}[section]
\newtheorem{lemma}[theorem]{Lemma}
\newtheorem{definition}[theorem]{Definition}

\newtheorem{corollary}[theorem]{Corollary}

\newtheorem{remark}[theorem]{Remark}
\newtheorem{claim}[theorem]{Claim}

\newcommand{\R}{\mathbb{R}}
\newcommand{\X}{\mathcal{X}}

\newcommand{\ty}{\mathbf{\tilde{y}}}

\newcommand{\lt}{\left}
\newcommand{\rt}{\right}

\newcommand{\zo}{\ensuremath{\{0,1\}}}

\newcommand{\dem}{W}
\newcommand{\demc}{\mathsf{dEMC}}
\newcommand{\dtc}{\mathsf{dCE}}

\newcommand{\eps}{\varepsilon}

\newcommand{\abs}[1]{\ensuremath \Bigl\lvert #1 \Bigr\rvert}

\newcommand{\fr}[1]{\ensuremath{\frac{1}{#1}}}

\newcommand{\ind}[1]{\ensuremath{\mathbf{1}[#1]}}

\DeclareMathOperator*{\E}{\mathbf{E}}

\newcommand{\Lip}{\mathsf{Lip}}

\newcommand{\smCE}{\mathsf{smCE}}

\newcommand{\sign}{\mathrm{sign}}

\newcommand{\Ber}{\mathrm{Ber}}

\newcommand{\calib}{\mathsf{Cal}}

\newcommand{\ECE}{\mathsf{ECE}}

\newcommand{\PX}{\mathsf{P}(\mathcal{X})}
\newcommand{\mQ}{\mathcal{Q}}

\newcommand{\uCE}{\overline{\mathsf{dCE}}}
\newcommand{\lCE}{\underline{\mathsf{dCE}}}
\newcommand{\dCE}{\mathsf{dCE}}

\newcommand{\cl}{\mathsf{Cal}}
\newcommand{\D}{\mathcal{D}}
\newcommand{\J}{\mathsf{Pld}}

\newcommand{\mS}{\mathcal{S}}

\newcommand{\mL}{\mathcal{L}}

\newcommand{\mC}{\mathcal{C}}
\newcommand{\mD}{\mathcal{D}}

\newcommand{\x}{\mathbf{x}}
\newcommand{\z}{\mathbf{z}}

\newcommand{\y}{\mathbf{y}}
\newcommand{\p}{\mathbf{p}}
\newcommand{\q}{\mathbf{q}}

\newcommand{\Pred}{PLD\xspace}
\newcommand{\pred}{PLD\xspace}

\newcommand{\preds}{PLDs\xspace}

\newcommand{\cdl}{\mathsf{CDL}}

\newcommand{\smce}{\mathsf{smCE}}
\newcommand{\ext}{\mathrm{lift}}

\usepackage{bm}
\usepackage[shortlabels]{enumitem}
\renewcommand{\abs}[1]{\lvert#1\rvert}

\usepackage{subcaption}
\usepackage{tabularray}
\usepackage{csquotes}
\begin{document}

\maketitle

\begin{abstract}
Recent work has highlighted the centrality of smooth calibration \cite{kakadeF08} as a robust measure of calibration error. 
We generalize, unify, and extend previous results on smooth calibration, both as a robust calibration measure, and as a step towards omniprediction, which enables predictions with low regret for downstream decision makers seeking to optimize some proper loss unknown to the predictor. 

\begin{itemize}
    \item We present a new omniprediction guarantee for smoothly calibrated predictors, for the class of all bounded proper losses. We smooth the predictor by adding some noise to it, and compete against smoothed versions of any benchmark predictor on the space, where we add some noise to the predictor and then post-process it arbitrarily. The omniprediction error is bounded by the smooth calibration error of the predictor and the earth mover's distance from the benchmark. We exhibit instances showing that this dependence cannot, in general, be improved. We show how this unifies and extends prior results \cite{FosterV98, HartlineWY25} on omniprediction from smooth calibration.
    
    \item We present a crisp new characterization of smooth calibration in terms of the earth mover's distance to the closest perfectly calibrated joint distribution of predictions and labels. This also yields a simpler proof of the relation to the {\em lower distance to calibration} from \cite{BlasiokGHN23}.
    
    \item We use this to show that the {\em upper distance to calibration} cannot be estimated within a quadratic factor with sample complexity independent of the support size of the predictions. This is in contrast to the {\em distance to calibration}, where the corresponding problem was known to be information-theoretically impossible: no finite number of samples suffice \cite{BlasiokGHN23}.
\end{itemize}

\end{abstract}
\thispagestyle{empty}
\newpage

\tableofcontents
\thispagestyle{empty}
\newpage

\setcounter{page}{1}

\section{Introduction}

Consider the setting of binary classification, where we wish to learn a predictor $p:\X \to [0,1]$  based on labeled samples $(\x, \y)$ drawn from a distribution $\mD$ on $\X \times \zo$. The prediction $p(x)$ represents our estimate of the conditional probability at $x$ that the label is $1$. Perfect \emph{calibration}, a classical notion which originates in the forecasting literature \cite{Dawid}, requires that $\E[\y|p(\x)] = p(\x)$.
Calibration guarantees several desirable properties to any downstream decision maker who uses these predictions to minimize a \emph{proper} loss function. These guarantees have the flavor of ensuring that the decision maker can {\em trust} the predictor, as if it were the Bayes optimal predictor. 

The property of calibration that is most relevant to our work is that it ensures no-regret with respect to all post-processings; this was first shown in the classic work of Foster and Vohra \cite{FosterV98}. Indeed, a predictor $p$ is calibrated if and only if $p$ has lower expected loss than $\kappa \circ p$ for all post-processing functions $\kappa : [0, 1] \to [0, 1]$ and all proper loss functions $\ell : [0, 1] \times \{0, 1\} \to \R$. Restated in the language of omniprediction (introduced in \cite{omni} and applied in the context of calibration by \cite{HartlineWY25, GopalanSTT26}), a calibrated predictor is a perfect $(\mL, \mC)$-omnipredictor, where the loss family $\mL$ comprises all proper loss functions, and the hypothesis class $\mC$ comprises all post-processings of $p$ itself.

Perfect calibration is not an achievable goal in practice for both computational and information-theoretic reasons. This has led to a long line of work  that aims at finding notions of approximate calibration that are more computationally tractable to achieve and to verify, and which preserve the desirable properties of calibration for decision making. Much of this work focuses on two main questions:
\begin{itemize}
    \item How should we measure the calibration error of our predictor, so as to have a measure that is both robust and efficient? This question has been studied in \cite{kakadeF08, KSJ18, BlasiokGHN23, when-does, BlasiokN24}.
    \item What kind of loss minimization guarantees for downstream decision makers can we get from approximate notions of calibration? This question has been studied in \cite{when-does, HuWu24, GopalanSTT26}.
\end{itemize}
We refer the reader to the recent survey \cite{GH-survey} for more details.

The work of \cite{BlasiokGHN23} suggested a property testing-inspired approach to the first question: measure the calibration error of predictor by the distance from the nearest perfectly calibrated predictor. They refer to this notion as the \emph{distance to calibration}. However, implementing this approach immediately runs into some basic questions: How do we measure distance between predictors? What family of perfectly calibrated predictors should we consider? Tackling these questions leads to a surprisingly subtle and intricate theory of distance to calibration \cite{BlasiokGHN23, GH-survey}; we will present formal statements later in the paper. In particular, the notion of \emph{smooth calibration error}, introduced in the early work of \cite{kakadeF08}, plays a key role in this theory. The main result of \cite{BlasiokGHN23} is that smooth calibration error is equivalent (up to constants) to a measure they call the {\em lower distance to calibration}, and this is information-theoretically the best efficient approximation to the distance to calibration that one can hope to achieve.

Smooth calibration has several desirable properties as a calibration measure: it is efficient to estimate and Lipschitz continuous in the predictions. However, smooth calibration by itself does not give the type of omniprediction guarantees that one gets from perfect calibration; there are fairly simple examples of smoothly calibrated predictors $p$ and proper losses where $p$ incurs significantly worse expected loss than some post-processing  $\kappa \circ p$. An elegant result of \cite{HartlineWY25} shows that this  situation can be remedied by adding some noise to the predictions before making decisions. They show that this results in an omniprediction guarantee for all proper loss functions, but where the hypothesis class is the space of all calibrated predictors for the distribution $\mD$, rather than post-processings. Another difference is that the quality of the omniprediction guarantee decays as the distance to the calibrated predictor being used as a baseline increases. Thus, smooth calibration gives some omniprediction guarantees, provided we add random noise to smooth the predictions. 

Other relaxations of perfect calibration that yield no-regret for post-processing have been studied in the literature. It is known that the expected calibration error ($\ECE$) is an upper bound on the regret \cite{KLST23}. Recently, Hu and Wu \cite{HuWu24} introduced the notion of calibration decision loss ($\cdl$)  which exactly captures this regret. But both these notions are known to be inefficient to estimate in the predictions-only access model  where we are given random samples $(p(\x), \y)$ for $(\x, \y) \sim \mD$ \cite{GopalanSTT26}. 

\subsection{Preliminaries}
\label{sec:intro-prelim}

We begin with a key new definition: \emph{prediction-label distributions (\preds)}.

\begin{definition}[Prediction-Label Distribution]
\label{def:pld-1}
    Consider the space $\mS = [0,1] \times \zo$, where we view the first coordinate as a predicted probability $p \in [0,1]$ and the second as a label $y \in \zo$. A \emph{Prediction-Label Distribution (\pred)} is a probability distribution $\mu$ over $\mS$. In the case that $\mu$ has a density function, which we also denote $\mu: \mS \to \R^+$, it satisfies
    \[ \sum_{y \in \zo} \int_0^1 \mu(p, y)dp  = 1.  \]
    We denote the set of all \preds by $\J$. The space $\mS \subseteq \R^2$ is equipped with the standard $\ell_1$ distance. Given distributions $\mu, \nu \in \J$, we denote the earth mover's distance between them by $\dem(\mu, \nu)$. (More precisely, $W(\mu, \nu) = \inf_\pi \E_{(\p,\y,\p',\y')\sim\pi}[d((\p,\y),(\p',\y'))]$, where the infimum is over all couplings $\pi$ of $\mu$ with $\nu$, and the metric is $d((p,y),(p',y'))=\lvert p-p'\rvert + \lvert y-y'\rvert$.)

\end{definition}

The notion of a \pred has been implicit in prior work on calibration (see for instance \cite{BlasiokGHN23}), but the terminology is new, and reasoning directly about this space, and the earth mover's distance over it will play a key part in our results. A typical \pred is depicted in Figure~\ref{fig:generic-pld}.

\begin{figure}
    \centering
    \includegraphics[height=1.5in]{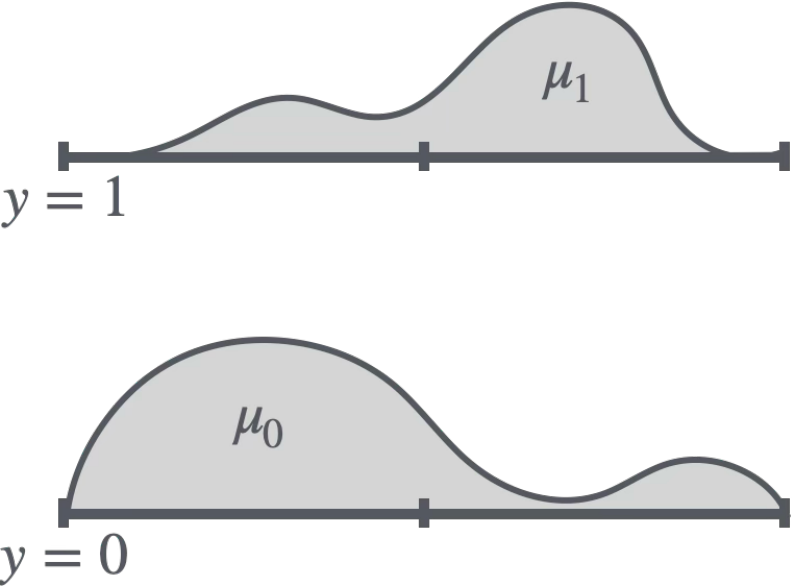}
    \caption{A typical \pred $\mu$, with $\mu_b = \Pr_{(\p, \y)\sim\mu}[\y=b]$. Note that $\mu_0 + \mu_1=1$.}
    \label{fig:generic-pld}
\end{figure}

\begin{definition}[Calibration]
\label{def:pld-2}
    The \pred $\mu$ is (perfectly) \emph{calibrated} if it satisfies the condition
    \[\mu(p,0)p = \mu(p,1)(1- p)\]
    for all $p \in [0,1]$.\footnote{Strictly speaking, we need this condition to hold for a set of measure $1$, but we will ignore measure-theoretic subtleties throughout this paper.}
    We denote the set of perfectly calibrated \preds by $\cl \subseteq \J$.
\end{definition}

Let $\PX = \{p: \X \to [0,1]\}$ denote the space of predictors on $\X$.  Given a distribution $\mD$ on $\X \times \zo$, we associate $p$ with a \pred $p \circ \mD$, defined to be the distribution of the prediction-label pair $(p(\x), \y)$ when $(\x, \y) \sim \mD$. We define $\calib(\mD) \subseteq \J$ to be the set of \preds of the form $p \circ \mD$, where $p \in \PX$ is perfectly calibrated for the distribution $\mD$, meaning that $\E[\y|p(\x)]=p(\x)$. We define $\tau(\mD) = \Pr_{(\x,\y)\sim\mD}[\y=1]$ to be the expected label under $\mD$. Given $\tau \in [0, 1]$, we define $\J^{=\tau} \subseteq \J$ to be the subset of \preds $\mu$ so that $\Pr_\mu[\y =1] = \tau$. We further define $\cl^{=\tau} = \J^{=\tau} \cap \cl$ to be the subset of calibrated \preds. We have
\[
    \cl(\mD) \subseteq \cl^{=\tau(\mD)} \subseteq \cl.
\]

Calibration error measures and loss functions are typically defined as the expectation of a suitable function on $\mS$ under the \pred  $p\circ \mD$. For instance, writing ($\p, \y) \sim \mu$ to denote sampling according to $\mu$, the expected calibration error (ECE) is defined as 
\[ \ECE(\mu) = \E_{(\p, \y)\sim \mu} \bigl\lvert\E[\y|\p] - \p\bigr\rvert,\]
whereas the smooth calibration error \cite{kakadeF08} is defined as
\begin{equation} 
\smCE(\mu) = \max_{\psi \in \Lip}\E_{(\p, \y) \sim \mu}[\psi(\p)(\y - \p)] \label{eq:smce}
\end{equation}
where $\Lip$ is the family of $1$-Lipschitz functions $\psi : [0, 1] \to [-1, 1]$.  Being a property only of \preds ensures that loss functions and calibration measures are defined uniformly for binary classification tasks across domains, regardless of whether the data are images, text, or numeric.  

\subsection{Prior Work}
To set the stage for our work, we discuss the most relevant prior work in detail.  

\subsubsection{Distance to Calibration}
Given a distribution $\D$ on $\X \times \zo$ and predictors $p, q:\X \to [0,1]$, we define the expected $\ell_1$ distance as
 \[ \ell^\D_1(p,q) = \E_{\x \sim \D}\bigl\lvert p(\x) - q(\x)\bigr\rvert.\]

\cite{BlasiokGHN23} define the {\em true distance to calibration} as the expected $\ell_1$ distance to the closest calibrated predictor in $\calib(\mD)$. However, it is hard to reason about such predictors without knowing the underlying space $\X$. In going from the predictor $p$ to the \pred $p\circ \mD$, we have lost information about the space $\X$ over which the predictor $p$ is defined, which makes measuring distance to calibration challenging. The solution proposed in \cite{BlasiokGHN23} was, given the joint distribution $\mu = p \circ \mD$, to consider all spaces $\X'$, predictors $p':\X' \to [0,1]$ and distributions $\mD'$ such that $p' \circ \mD' = \mu$. By considering the minimum and maximum distance to calibration over all such spaces, they define two new calibration measures, the lower and upper distance to calibration, which sandwich the true distance to calibration between them. Formally:

\begin{definition}
 \label{def:dtc}
 Given a \pred $\mu$, let $\ext(\mu)$ be the set of all predictor-distribution pairs $(p',\D')$ such that $p'\circ\D' = \mu$.\footnote{
 The underlying {\em feature space} $\X'$ over which $p'$ and $\D'$ are defined can vary for the elements in $\ext(\mu)$.}
Given a predictor $p:\X \to [0,1]$ and a distribution $\D$ on $\X \times \zo$,
\begin{itemize}
\item The (true) distance to calibration is defined as
 \[ \dtc_\D(p) = \inf_{\substack{q : \X \to [0, 1] \\ q \circ \mD \in \cl(\D)}}\ell_1^\D(p,q). \]

 \item The upper distance to calibration is defined as
 \[ \uCE(p \circ \D) = \sup_{(p', \D') \in \ext(p \circ \D)}\dtc_{\D'}(p').\]

\item The lower distance to calibration is defined as 
 \[ \lCE(p \circ \D) = \inf_{(p', \D') \in \ext(p \circ \D)}\dtc_{\D'}(p').\]
\end{itemize}
\end{definition}

Observe that unlike $\dCE$, both $\uCE$ and $\lCE$ are only functions of the \pred $p \circ \mD$, and not the underlying domain $\X$. From the definitions, it follows that
\[ \lCE(p\circ D) \leq \dtc_\mD(p) \leq \uCE(p \circ \mD). \]
\cite{BlasiokGHN23} ask how easy it is to compute each of the quantities $\lCE$, $\dCE$, and $\uCE$ from sample access to $p \circ \mD$. They show that the notion of smooth calibration error from \cite{kakadeF08} (defined in Equation \eqref{eq:smce}) holds the key to the answer.

\begin{theorem}[\cite{BlasiokGHN23}]
\label[theorem]{thm:ldce-smce}
    \(
        1/2 \le \underline{\dtc}(\mu)/\smce(\mu) \le 2
    \)
    for every \pred $\mu$.
\end{theorem}

Further, \cite{BlasiokGHN23} show that the upper and lower distance are within a quadratic factor of each other:
\begin{equation} 
\label{eq:lower-upper}
\uCE(\mu) \leq 4\sqrt{\lCE(\mu)},
\end{equation}
and that it is information theoretically impossible to estimate the true distance better than within a quadratic factor. This tells us that $\smCE(p \circ \mD)$ is essentially the best approximation one can get to both the true distance $\dtc_\mD(p)$, and the lower distance $\lCE(p \circ \mD)$ (up to constant factors). They leave open the question of whether we can compute the upper distance $\uCE(p \circ \mD)$ any better.

\subsubsection{Omniprediction From Smooth Calibration.}

A (bounded) loss function is a function $\ell:\mS \to [-1,1]$, and a proper loss is one where if the label $\y \sim \Ber(p)$, then $\E[\ell(q, \y)]$ is minimized at $q = p$. We let $\mL$ denote the family of all bounded proper losses, and define the notion of omnipredictors, specialized to this family.

\begin{definition}[Omnipredictor, \cite{omni}]
    Let $\D$ be a distribution on $\X \times \zo$, let $\mL$ be the family of bounded proper losses, and $\mQ \subseteq \PX$ be a set of predictors.  
    A predictor $p:\X \to [0,1]$ is an $(\mL, \mQ, \alpha)$-omnipredictor if for every $\ell \in \mL$ and $q \in \mQ$, it holds that
    \[ \E[\ell(p(\x), \y)] \leq \E[\ell(q(\x), \y)] + \alpha.\]
\end{definition}
The setting $\alpha = 0$ corresponds to perfect omniprediction, where $p$ is optimal compared to $\mQ$, while $\alpha =2$ is trivial since our losses are bounded in $[-1,1]$. So we will  consider $\alpha \in [0,2)$. We are interested in settings where one can have $\alpha$ bounded away from $2$, ideally tending to $0$.\footnote{Like calibration, expected loss is a property of a \pred, rather than the predictor itself. However, in omniprediction, the underlying distribution $\mD$ is fixed, which fixes the marginal distribution on $\y$. Hence, we will think of omniprediction both as a property of \preds that lie in $\J^{=\tau(\mD)}$, and of predictors $p \in \PX$, with the associated \pred $p\circ \mD$.}

Foster and Vohra \cite{FosterV98} showed that a perfectly calibrated predictor $p$ is a perfect omnipredictor for all proper losses.

\begin{theorem}[\cite{FosterV98}]
\label{thm:fv}
    Let $p$ be perfectly calibrated. Let $K = \{\kappa\mid\kappa: [0,1] \to [0,1]\}$  be  all possible post-processings. Then for any $\ell \in \mL$, and $\kappa \in K$ we have $\E[\ell(\p, \y)] \leq \E[\ell(\kappa(\p), \y)]$.\footnote{They prove this for the squared loss, but it extends straightforwardly to all proper losses.} 
\end{theorem}

One cannot replace perfect calibration with smooth calibration in their statement: For every $\eps > 0$, we will give an example of a \pred $\mu$ with $\smCE(\mu) \leq \eps$, a loss $\ell \in \mL$, and post-processing $\kappa$ so that
\[ \E_\mu[\ell(\p, \y)] = 1, \quad \E_\mu[\ell(\kappa(\p), \y)] = 0,\] 
This raises the question of whether it is at all possible to show omniprediction guarantees against a rich class of post-processings for a smoothly calibrated predictor. 

A novel omniprediction guarantee for smooth calibration error was shown in the recent work of \cite{HartlineWY25}. There are two important ways in which their guarantee differs from the Foster-Vohra result:
\begin{enumerate}
\item {\bf Smoothed predictions.} Rather than using the predictions $\p$ of a smoothly calibrated predictor directly, they add some noise $\z$ sampled uniformly from the interval $[-\sigma, \sigma]$ and truncate the result to $[0,1]$.  We will denote this noise distribution by $\z \sim [\pm \sigma]$ and the truncated predictor by $\p_\z$.
\item {\bf The baseline class.} The omniprediction guarantee holds for the baseline class of all calibrated \preds $\mu \in \calib^{=\tau(\mD)}$, which satisfy accuracy in expectation. This is a rich class that includes the  Bayes optimal \pred $\mu^* = (\E[\y|\x], \y)$, and it is impossible to hope for an omniprediction guarantee with any fixed $\alpha < 1$. Instead, they allow $\alpha$ to degrade with the earth mover's distance between $p \circ \mD$ and $\mu$.  
\end{enumerate}

Formally, \cite{HartlineWY25} show that smoothing $p$ gives an omnipredictor for $\calib^{=\tau(\mD)}$ with the following guarantee:\footnote{
The exact statement of Theorem \ref{thm:hwy} does not appear in \cite{HartlineWY25}, but it can be derived from Theorem 3.1 therein. There are some differences between their statement and ours. They sample the noise from a suitably chosen DP mechanism (Gaussian or Laplace), but we prefer random shifts because they are simpler, and as we will see, they correspond to random bucketing which is commonly used in practice. Although their bound is not stated in terms of $W(\cdot, \cdot)$, the way we have stated it is equivalent by \cref{thm:old-ldce-as-w}. }
\begin{theorem}[\cite{HartlineWY25}]\label{thm:hwy}
    
    Let $p \in \PX$, $\mD$ be a distribution on $\X \times \zo$  and $\nu \in \cl^{= \tau(\mD)}$. For $\ell \in \mL$ and $\sigma \in (0,1]$, we have
    \[ \E_{\substack{(\p,\y) \sim p \circ \D, \\ \z \sim [\pm \sigma]}}[\ell(\p_\z, \y)] \leq \E_{(\q, \y)\sim \nu}[\ell(\q, \y)] + O\lt(\sigma + \fr{\sigma} W(\p \circ \mD,\nu)\rt).\]
\end{theorem}
To connect this to smooth calibration, we observe the smallest that the earth mover's distance can be is $\lCE(p) = \Theta(\smCE(p))$, by the result of \cite{BlasiokGHN23}. We can take $\sigma = \sqrt{\smCE(p\circ \mD)}$, to get an omniprediction error bound of $\alpha = O(\sqrt{\smCE(p\circ \mD)})$. Thus the guarantee is particularly meaningful for smoothly calibrated predictors, where $\alpha$ goes to $0$. Since smooth calibration error can be efficiently estimated, one can estimate the level of noise to add to the predictor before making decisions.

Given these two incomparable results, there are two natural questions that arise:
\begin{itemize}
    \item Can one extend the Foster-Vohra guarantee (Theorem \ref{thm:fv}) and show an omniprediction guarantee for smoothly calibrated predictors, against the class of post-processings?
    \item Can one relax the assumption of calibration on $\q$ in Theorem \ref{thm:hwy} and show an omniprediction guarantee against all \preds in $\J^{=\tau(\mD)}$?
\end{itemize}

\subsection{Our Results}

In this work, we prove some new results and reprove some old results about the properties of smooth calibration, both as a replacement for perfect calibration that gives strong omniprediction guarantees, and as a measure of the distance to calibration. \begin{itemize}
    \item We present a general omniprediction result for smooth calibration, that strengthens and generalizes the Foster-Vohra result by allowing post-processings and the \cite{HartlineWY25} result by comparing with  all nearby predictors, whether or not they are calibrated.
    \item We present a crisp new characterization of the lower distance to calibration in terms of earthmover distance from the set $\cl$. We use this to present an arguably simpler proof of the tight connection between smooth calibration error and the lower distance to calibration.
    \item We show that the upper distance to calibration cannot be estimated to better than a quadratic factor, with sample complexity independent of the support size of the predictor.
\end{itemize}

Common to all our results is a view of calibrated predictors on a space being sandwiched between the larger set of all calibrated distributions on prediction-label pairs, and the  smaller set of calibrated post-processings, which is loosely inspired by convex relaxations in discrete optimization. This view informs our formulation of theorem statements, we believe it also results in more general and clarifying proofs.

\section{Technical Overview}

In this section, we present full technical statements of our results and the intuition behind them, without getting into proof details. 

\subsection{Omniprediction From Smooth Calibration.}\label{section:technical-overview-omniprediction}
We show an omniprediction guarantee for smoothly calibrated predictors, against the baseline class of post-processings of predictors in $\J^{=\tau(\mD)}$.\footnote{The guarantee stated would also hold for $\nu \in \J$, but changing the label distribution is not natural for omniprediction.} 
We write $f \lesssim g$ if $f = O(g)$ and $p_z$ for $[p + z]_0^1$, where $[\cdot]_0^1$ denotes projection onto the interval $[0, 1]$.

\begin{restatable}{theorem}{OmniMain}\label{theorem:omniprediction-intro}
Let $p \in \PX$, $\mD$ be a distribution on $\X \times \zo$, $\mu = p \circ \mD$ and $\nu \in \J^{=\tau(\mD)}$.
For any bounded proper loss $\ell \in \mL$, post-processing function $\kappa:[0,1] \to [0,1]$, and $\sigma \in (0,1]$,  
\[
        \E_{\substack{(\p, \y) \sim \mu \\ \z \sim [\pm \sigma]}}[ \ell(\p_\z,\y)] - \E_{\substack{(\q, \y) \sim \nu\\ \z \sim [\pm \sigma]}}[\ell(\kappa(\q_\z),\y)] \lesssim \sigma + \frac{1}{\sigma}\Bigr( \smce(\mu) + W(\mu, \nu)\Bigr). 
    \]
\end{restatable}

Our theorem allows for arbitrary \preds $\nu \in \J^{=\tau(\mD)}$, allowing the guarantee to decay with $W(\mu, \nu)$. It also allows for arbitrary post-processings $\kappa$. Thus it gives a common generalization of the baseline classes used in \cite{FosterV98} and \cite{HartlineWY25}. 

In the setting where $\q = \p$, our result implies the following bound: 

\begin{restatable}{corollary}{cor:qp}
\label{corollary:smooth-cdl}
    Let $(\p,\y)\sim\mu$ and let $\z$ be uniform over $[-\sigma,\sigma]$ and independent from $(\p,\y)$. Then, the following is true for any proper loss $\ell:[0,1]\times\{0,1\}\to [-1,1]$ and any $\kappa:[0,1]\to [0,1]$:
    \begin{align}
    \label{eq:q=p}
\E_{\substack{(\p, \y) \sim \mu \\ \z \sim [\pm \sigma]}}[ \ell(\p_\z,\y) - \ell(\kappa(\p_\z), \y)]  \lesssim \sigma + \frac{\smCE(\mu)}{\sigma}.  
    \end{align}
 \end{restatable}

This can be interpreted as saying that smoothly calibrated predictors are omnipredictors with respect to arbitrary post-processings, provided we add noise to the predictions. This gives a {\em smoothed analysis} analogue for the Foster-Vohra result for smooth calibration. It is essential to smooth the comparison baseline $\q$ with random noise, the above statement is not true if we replace $\kappa(\p_\z)$ with $\kappa(\p)$, and only assume that $p$ is smoothly calibrated. 

A similar statement may be derived by combining the results of Blasiok and Nakkiran \cite{BlasiokN24}, showing that adding noise (from a different distribution) to a smoothly calibrated predictor yields small $\ECE$, with the results of \cite{KLST23, HuWu24}, showing that small $\ECE$ suffices for omniprediction guarantees against post-processings. 

In the setting of \cite{HartlineWY25}, where we restrict the benchmark class to calibrated \preds from $\cl^{=\tau(\mD)}$, we do not need to smooth $\q$, and can show a quantitatively stronger bound:

\begin{restatable}{theorem}{OmniCal}\label{theorem:omniprediction-intro2}
Let $p \in \PX$, $\mD$ be a distribution on $\X \times \zo$, $\mu = p \circ \mD$ and $\nu \in \cl^{=\tau(\mD)}$.
For any bounded proper loss $\ell \in \mL$, and $\sigma \in (0,1]$  
\[
        \E_{\substack{(\p, \y) \sim p\circ \mD \\ \z \sim [\pm \sigma]}}[ \ell(\p_\z,\y)] - \E_{(\q, \y) \sim \nu}[\ell(\q,\y)] \lesssim \sigma + \frac{\smce(\mu)}{\sigma} + W(\mu, \nu).
\]
\end{restatable}
Since $\nu$ is calibrated, post-processing does not reduce the loss. Hence the same bound holds for all post-processings $\kappa(\q)$. 
In comparison to Theorem \ref{thm:hwy}, our result replaces the term $W(\mu, \nu)/\sigma$ with $\smCE(\mu)/{\sigma} + W(\mu, \nu)$. Corollary \ref{cor:bghn-redux} shows that $\smCE(\mu)$ is within constant factors of the minimum value of $W(\mu, \nu)$ over all $\nu \in \calib^{=\tau(\mD)}$. For such $\nu$, the two bounds are similar up to constants. But when $W(\mu, \nu) \gg \smCE(\mu)$, we incur the smaller penalty of $W(\mu, \nu)$ rather than $W(\mu, \nu)/\sigma$, thus improving over Theorem \ref{thm:hwy}. 

Our proof uses techniques from the literature on loss outcome-indistinguishability, introduced in \cite{lossOI} and used in the calibration context by \cite{GopalanSTT26}. However, we believe a key contribution is identifying the right statement itself, which is rather delicate. To illustrate this, consider a simple example.

Consider a two point space $\X = \{x_0,x_1\}$. The distribution $\mD$ is uniform on $(x_0,0)$ and $(x_1, 1)$. We will consider the {\em proper} version of 0-1 loss, $\ell_{\zo}:[0,1] \times \zo \to \zo$ where 
\[ \ell_{\zo}(p,y) = |\mathrm{1}(p \geq 1/2) - y|.\]
In other words, if we are on the correct side of $1/2$, we suffer $0$ loss, else the loss is $1$. This is essentially a special case of the $v$-shaped losses which are studied in the literature. 
Now consider the following predictors:
\begin{itemize}
    \item The {\em good} predictor $g$ where $g(x_0) = 1/2 - \eps, g(x_1) = 1/2 + \eps$, 
    which has expected loss $0$.
    \item The {\em bad} predictor $b$ where $b(x_0) = 1/2 + \eps, b(x_1) = 1/2 - \eps$,
    which has expected loss $1$.
  \item The {\em uniform} predictor
    $u$ where
    $u(x_0) = u(x_1) = 1/2$,
    which has expected loss $1/2$.
\end{itemize}
Smooth calibration error does not distinguish between good and bad (or even uniform), since  $\smCE(b) = \smCE(g) = \eps$, while $u$ is perfectly calibrated.  We observe that for $\z \in [\pm \sigma]$, $\mathrm{dTV}(g_\z, b_\z) \leq 2\eps/\sigma$. Finally, $g = \kappa(b)$ where $\kappa(p) = 1 - p$. With these observations in hand, we can deduce the following (see \Cref{theorem:smoothing-necessity}):

\begin{itemize}
    \item The gap between expected loss of $p$ and $\kappa(p)$ can be $1$, even when $\smCE(p) \leq \eps$. To see this, take $p = b$, $q = \kappa(p) = g$. Thus one cannot hope for a direct analogue of the Foster-Vohra bound for smooth calibration without any smoothing. 
    \item The same example shows that the gap between the expected loss of $p_\z$ and $q = \kappa(p)$ can be $1/2$, even when $\smCE(p) \leq \eps$, and $W(p\circ \mD, q \circ \mD) \leq \eps$. Note that when $\sigma \gg \eps$, $p_\z$ is essentially the same as $u$, so it has expected loss $1/2$. Thus smoothing $p$ alone but not smoothing $q$ is not sufficient. 
    \item By taking $p = b$ and $q  = g$, we see that the gap between the expected loss of $p$ and $q_\z$ can be $1/2$, since now $q_\z$ is essentially the same as $u$ for $\sigma \gg \eps$. So smoothing $q$ alone without smoothing $p$ is also not sufficient.  
\end{itemize}

Further, in \Cref{theorem:lower-bound-1,theorem:lower-bound-2}, we give examples which illustrate that the omniprediction error must scale linearly with each of the three terms on the RHS of Theorem \ref{theorem:omniprediction-intro}.

\subsection{Lower Distance to Calibration}
\label{sec:dtc-results}

Our next set of results uses what we call the \emph{earth mover's distance to calibration}, denoted $\demc$, to prove new results regarding the upper and lower distance to calibration, as well as provide new proofs of existing results. We begin with the definition of $\demc$:
\begin{definition}[$\demc$]
   Given $\mu \in \J$, the \emph{earth mover's distance to calibration} is
   \[
    \demc(\mu) = \inf_{\nu \in \calib}\,W(\mu,\nu).
   \]
\end{definition}

To begin, we use $\demc$ to give a simple proof of \cref{thm:ldce-smce}, which states that the smooth calibration error approximates the lower distance to calibration up to a constant factor. This result originally appeared in \cite{BlasiokGHN23}, with an alternate proof in \cite{BlasiokN24}. Next, we show that the upper distance to calibration is hard to approximate within a quadratic factor in the \emph{prediction-only access model}, where we have access to samples of the form $(\p, \y)$. Specifically,  \cref{thm:udce-hard-intro} is a sample complexity lower bound that scales with $\Omega(\sqrt{k})$, where $k$ is the support size of the prediction distribution. This result complements prior work showing that the (ordinary) distance to calibration is \emph{impossible} to estimate from samples within a quadratic factor in the same model \cite{BlasiokGHN23}.

\paragraph{Lower Distance and Earth Mover's Distance}

Our new proof of \cref{thm:ldce-smce} is based on the following two lemmas. The first lemma relates smooth calibration error to the earth mover's distance to calibration, which the second lemma in turn equates to the lower distance to calibration. Interestingly, the second lemma can be viewed as an alternate and natural \emph{definition} of the lower distance to calibration.

\begin{restatable}{lemma}{EmcSmce}
\label[lemma]{thm:emc-smce}
    \(
        1/2 \le \demc(\mu)/\smce(\mu) \le 2
    \)
    for every \pred $\mu$.
\end{restatable}

\begin{restatable}{lemma}{LdceEmc}
\label[lemma]{thm:ldce-emc}
    \(
        \lCE(\mu) = \demc(\mu)
    \)
    for every \pred $\mu$.
\end{restatable}

It is clear that these two results yield \cref{thm:ldce-smce} when combined. It remains to prove them. The crucial step in the proof of \cref{thm:emc-smce} uses the apparent flexibility in the definition of $\demc$, which allows one to change the $\y$ distribution when designing a nearby calibrated predictor. In particular, given $\mu =(\p, \y)$, it is easy to see that $\nu =(\p,\ty)$ where $\ty \sim \Ber(\p)$ is calibrated, and $W(\mu, \nu) = \Theta(\smCE(\mu))$ (see \cite[Lemma 3.4]{GH-survey}). The proof of \cref{thm:ldce-emc}, in contrast, shows that this flexibility in the $\y$ distribution, while convenient, was not actually necessary and that we can restrict $\nu$ to lie in the set $\cl^{=\tau}$ where  $\tau = \Pr_\mu[\y =1]$.

In contrast to our proof strategy, which centers on the earth mover's distance to the set $\calib$, it turns out the prior proof of \cite{BlasiokGHN23} was effectively considering $\calib^{=\tau}$. 
To see this connection, we use a helpful interpretation of $\lCE$, proposed in \cite{BlasiokGHN23}, as the following infimum:
\begin{lemma}[\cite{BlasiokGHN23}]
\label[lemma]{thm:triples}
    $\lCE(\mu) = \inf \E\lvert \p-\q\rvert$, where the infimum is taken over all triples of joint random variables $(\p, \q, \y)$ such that $(\p, \y) \sim \mu$ and $(\q, \y)$ is calibrated. 
\end{lemma}

The next lemma (proved in \cref{sec:ldce-smce}) shows how to move from such triples to earth mover's distance.

\begin{restatable}{lemma}{OldLdceAsW}
\label[lemma]{thm:old-ldce-as-w}
For any two $\mu,\nu\in\J^{= \tau}$ for $\tau \in [0,1]$, we have that $W(\mu,\nu) = \inf \E[\lvert\p-\q\rvert]$, where the infimum is taken over all triples of joint random variables $(\p,\q,\y)$ such that $(\p,\y)\sim\mu$ and $(\q,\y)\sim\nu$. 
\end{restatable}

Putting these two lemmas together, it follows that the prior characterization of $\lCE$ can be viewed as the earth mover's distance to $\calib^{=\tau}$ for an appropriate choice of $\tau$:

\begin{corollary}
\label{cor:bghn-redux}
    $\lCE(\mu) = \inf_{\nu \in \calib^{=\tau}}\, W(\mu, \nu)$, where 
    $\tau = \Pr_{(\p,\y)\sim\mu}[\y=1]$. 
\end{corollary}

As the proof in the present paper shows, working with the earth mover's distance to $\calib$ \emph{directly}, rather than to $\calib^{=\tau}$, is equivalent and leads to a cleaner picture. Furthermore, the original proof of the lower bound of \Cref{thm:ldce-smce} from \cite{BlasiokGHN23} used an argument based on linear programming duality, which in hindsight resembles a modified version Kantorovich-Rubinstein duality for earth mover's distance. Our proof seems to circumvent this argument by directly using Kantorovich-Rubinstein duality, which we of course need not rederive, in the proof of \cref{thm:emc-smce}. The results of this section also yield the following interesting corollary, which we state purely in terms of earth mover's distance to calibration.

\begin{restatable}{corollary}{DemcEquiv}
\label[corollary]{corollary:demc-equiv}
    Given a \pred $\mu$, the following are equal up to constant factors:
    \begin{enumerate}[(a)]
        \item Its earth mover's distance to calibration, $\demc(\mu)$:
        \[
             \inf \,\bigl\{ W(\mu, \nu) \,\big|\, \text{$\nu \in \calib$}\bigr\},
        \]
        \item Its earth mover's distance to calibration while preserving the marginal of $\y$:\
        \[
            \inf \,\bigl\{ W(\mu, \nu) \,\big|\, \text{$\nu \in \calib^{=\tau}$}\bigr\},
        \]
        where $\tau = \Pr_{(\p,\y)\sim\mu}[\y=1]$.
        \item Its earth mover's distance to calibration while preserving the marginal of $\p$:
        \[
            \inf \,\bigl\{ W(\mu, \nu) \,\big|\, \text{$\nu \in \calib$ and $\Pr_{(\p, \y) \sim \mu}[\p \le t] = \Pr_{(\p, \y) \sim \nu}[\p \le t]$ for all $t \in [0, 1]$}\bigr\}.
        \]
    \end{enumerate}
    More precisely,
    \(
        a = b \le c \le 2a.
    \)
\end{restatable}

\subsection{Intractability of the Upper Distance}
We have already defined the sets $\J$ and $\J^{=\tau(\mD)}$, which are supersets of the set $\J(\mD)$ of \preds induced by $\PX$. We now define the set of post-processings of a predictor, which are a subset of the space $\PX$ that will be relevant for our discussion of upper distance to calibration.

\begin{definition}[Post-processings of predictors and \preds]
Let $K = \{\kappa:[0,1] \to [0,1]\}$ denote the set of all possible post-processing functions. Let $K(p) = \{\kappa \circ p\}_{\kappa \in K}$ denote all post-processings of a predictor $p$, and observe that $K(p) \subseteq \PX$. Given a \pred $\mu = (\p, \y)$, let $K(\mu)$ denote the set of all \preds of the form $(\kappa(\p), \y)_{\kappa \in K}$, and let $\calib(\mu) = \calib \cap K(\mu)$. 
\end{definition}

We have the following inclusions among the set of \preds for any predictor $p$ and distribution $\mD$:
\[  \J(\mD) \supseteq K(p \circ \mD), \ \calib(\mD) \supseteq \calib(p \circ \mD).\]
\eat{
The sets $\J, \J^{=\tau}$ and $\calib, \calib^{=\tau}$ are convex sets, and are nicer to reason about.  When we wish to prove a claim about every \pred in $\J(\mD)$, it is typically easier to reason about one of the larger sets $\J$ or $\J^{=\tau(\mD)}$. $\J$ comes equipped with a natural metric, which is the earth mover's distance on distributions on $S$ (with the underlying metric on $S$ being $\ell_1$). In  this paper, we will show that the earth mover's distance to $\calib$ is central to the theory of distance to calibration. }
In order to prove the existence of a \pred  in $\J(\mD)$ with a certain property, one typically proves the existence of a \pred in $K(p \circ \mD)$. This is equivalent to only considering predictors in $K(p)$ rather than all predictors from $\PX$. However, $K$ is not a particularly tractable set, since it consists of all post-processing functions. Optimizing over it efficiently can be hard, as was  recently shown in the context of Calibration Decision Loss by \cite{GopalanSTT26}. We show a similar barrier to efficiently estimating the upper distance to calibration. 

We show that accurate estimates of the upper distance to calibration, which are asymptotically better than what is implied by the quadratic relation to the lower distance (Equation \eqref{eq:lower-upper}) are hard to obtain in the prediction-only access model. Specifically, we prove a sample complexity lower bound that scales with the size of the support of the distribution, which may be very large, even if finite. Our result complements prior work showing that the (ordinary) distance to calibration is not just hard, but \emph{impossible} to estimate in this model with the same level of accuracy. As we shall soon see, both of these results are most easily understood from the perspective of the earth mover's distance to calibration.

A key concept in the proof of the result of this section, as well as later results in this paper, is the following \emph{almost balanced} \pred. The distribution corresponds to a predictor which always predicts $1/2\pm\eps$. The direction of the deviation from $1/2$, however, is uncorrelated with the true label, which is a uniformly random bit.

\begin{definition}[Almost Balanced \Pred]
\label{def:almost-balanced}
    Given $\eps > 0$, the \emph{$\eps$-almost balanced \pred} is the distribution on $\mathcal{S} = [0, 1] \times \{0, 1\}$ with mass $1/4$ on each of $(1/2-\eps, 0)$, $(1/2-\eps, 1)$, $(1/2+\eps, 0)$, and $(1/2+\eps, 1)$.
\end{definition}

Our next result will construct several examples, each comprising a domain $\mathcal{X}$, a predictor $p : \mathcal{X} \to [0, 1]$, and a distribution $\mathcal{D}$ over $\mathcal{X} \times \{0, 1\}$. Roughly speaking, in each example, the prediction-label distribution $p \circ \mathcal{D}$ is similar to the $\eps$-almost balanced \pred, but distinguishing the scenarios from samples of the form $(p(\x), \y)$ may be hard. Our proof of the theorem will make use of intuitive visual arguments enabled by the earth mover's perspective.
Before we state the theorem, we briefly recall the following view of $\uCE$ in terms of $\calib(\mu)$ from \cite{BlasiokGHN23}:
\[
    \uCE(\mu) = \inf_\kappa \,\E_{(\p, \y)\sim\mu}\abs{\kappa(\p)-\p},
\]
where $\kappa:[0,1]\to[0,1]$ ranges over all calibrated post-processings, meaning $\E[\y|\kappa(\p)]= \kappa(\p)$.

\begin{theorem}[Succinct Version of \Cref{thm:udce-hard}]
\label{thm:udce-hard-intro}
    Fix parameters $\eps > 0$, $k \in \mathbb{N}$, and $\delta_{ij} \in\R$ for $(i,j) \in \{0,1\}\times[k]$. There exist four $4$-tuples $(\X_a, \D_{\X_a}, p^*_a, p_a)$, $(\X_b, \D_{\X_b}, p^*_b, p_b)$, $(\X_c, \D_{\X_c}, p^*_c, p_c)$, and $(\X_d, \D_{\X_d}, p^*_d, p_d)$ with the following properties. First, $|\X_c|=|\X_d|=k$. Second, cases (a) and (b) do not depend on the $\delta_{ij}$ parameters. Finally, if $\delta_{ij}$ are sampled i.i.d. from an appropriate continuous distribution, then:
    \begin{enumerate}[label=(\roman*)]
        \item Case~(a), which has $\dCE \ge \Omega(\eps)$, is impossible to distinguish with any positive advantage from case~(b), which has $\dCE \le O(\eps^2)$, given any finite number of prediction-label samples.
        \item It requires at least $\Omega(\sqrt{k})$ prediction-label samples to distinguish (with constant advantage in expectation over the choice of $\delta_{ij}$) case~(c), which has $\uCE \ge \Omega(\eps)$, from case~(d), which has $\uCE \le O(\eps^2)$ (with probability $1$ over the choice of $\delta_{ij}$).
    \end{enumerate}
In particular, $\dCE$ cannot be estimated within a better-than-quadratic factor from prediction-label samples, and $\uCE$ cannot be estimated within a better-than-quadratic factor from any number of prediction-label samples that is independent of the support size of the distribution of predictions.
\end{theorem}

Observe that part (i) of \cref{thm:udce-hard} is precisely the prior result of \cite{BlasiokGHN23} regarding the inapproximability of $\dCE$ in the prediction-only access model. Part (ii) is our new result regarding $\uCE$. Working with \preds and earth mover's distance greatly simplifies the task of finding explicit examples which exhibit separations between various measures; rather than conjuring up a clever space $\X'$, we simply find \preds in $\J$ or $K(p \circ \mD)$ that exhibit the desired separation.

\section{Omniprediction From Smooth Calibration}

In this section, we prove Theorems \ref{theorem:omniprediction-intro} and \ref{theorem:omniprediction-intro2}. 
Recall that $\z \sim [\pm \sigma]$ denotes $\mathbf z \sim \mathrm{Unif}[-\sigma,\sigma]$. Both theorems consider the smoothed predictor 
$p_{\mathbf z} = [p+\mathbf z]_0^1$, 
An equivalent way to interpret the smoothing operation 
is as follows. Draw $\mathbf w \sim \mathrm{Unif}[0,2\sigma]$ and round $p$ to the nearest point in the randomly shifted grid
\[
\mathcal I_{\mathbf w}
= \{0,1\} \cup \{[\mathbf w + 2 i \sigma]_0^1 : i \in \mathbb{Z}\}.
\]
For any fixed $p \in [0,1]$, the two procedures produce random variables with the same distribution.

We start with some simple lemmas.

\begin{lemma}
\label{lemma:tv}
    Let $z\in[-\sigma,\sigma]$ and $p \in [0, 1]$. Then we have $d_{\mathrm{TV}}(\mathrm{Bern}(p), \mathrm{Bern}(p_z)) \le \sigma$.
\end{lemma}
We skip the simple proof. The next lemma shows that adding noise makes bounded functions Lipschitz. 

\begin{lemma}
\label{lemma:smoothing}
    If $f : [0, 1] \to [-1, 1]$ is measurable and $\z$ is uniform in $[-\sigma,\sigma]$, then the function $f_\sigma$ with $f_\sigma(p) = \E[f(p_\z)]$ is $O(1/\sigma)$-Lipschitz in $p$.
\end{lemma}

\begin{proof}
    Let $p,q\in[0,1]$. If $\z\sim [-\sigma,\sigma]$, then the total variation distance between the random variables $p_\z$ and $q_\z$ is $O(|p-q|/\sigma)$. This is because the density of $p+\z$ is $D_1(t) = \ind{|t-p|\le \sigma}/2\sigma$ and the density of $q+\z$ is $D_2(t) = \ind{|t-q|\le \sigma}/2\sigma$. Therefore:
    \[
        d_{\mathrm{TV}}(p+\z, q+\z) = \frac{1}{2} \int_{t=-\infty}^\infty |D_1(t) - D_2(t)| \, dt \lesssim \frac{|p-q|}{\sigma}.
    \]
    The result follows from the fact that post-processing (i.e., clipping and applying $f$) does not increase the total variation distance.
\end{proof}

\subsection{Proof of Theorem \ref{theorem:omniprediction-intro}}

We now prove our main omniprediction result for smooth calibration, which we restate below. 

\OmniMain*

We use the notation $\y^{p}$ to denote the random variable following the Bernoulli distribution with probability of success $p$. Our proof will follow the loss OI technique introduced by \cite{lossOI}, where we start from a label distribution where the desired omniprediction guarantee holds true by Bayes optimality, and then bound the cost of modifying the label distribution. Since the predictor we wish to show omniprediction guarantees for is $\p_\z$, the appropriate label distribution is $\y^{\p_\z}$ where $\E[\y^{\p_\z}|\p_\z] = \p_\z$. Switching between $\y^{\p_\z}$ and $\y$ is enabled by the following lemma:

    \begin{lemma}\label{lemma:main1}
        Let $\ell:[0,1] \times \zo \to [-1,1]$ be a bounded, but not necessarily proper loss function. Then
        \begin{align}    \label{eq:smooth0}
        \lt|\E[\ell(\p_\z,\y)] - \E[\ell(\p_\z,\y^{\p_\z})]\rt| \lesssim \sigma + \frac{\smCE(\mu)}{\sigma}. 
        \end{align}
    \end{lemma}
    \begin{proof}
    Following \cite{lossOI}, defining $\partial \ell(p) = \ell(p, 1) - \ell(p,0)$, we can write
    \[ \ell(p,y) = \ell(p, 0) + y\partial \ell(p).\]
    Define the function $w(p) = \E_\z[\partial \ell(p_\z)]$. Since $|\ell| \leq 1$, $|\partial \ell| \leq 2$. Hence Lemma \ref{lemma:smoothing} implies that $w(p)$ is $O(1/\sigma)$-Lipschitz. Hence we have
    \begin{align}
    \label{eq:smooth1}
    \bigl\lvert\E[\ell(\p_\z,\y)]-\E[\ell(\p_\z,\y^{\p})]\bigr\rvert = \E[\partial \ell(\p_\z)(\y - \y^p)] = \E[w(\p)(\y - \p)] \leq O\lt(\frac{\smCE(\mu)}{\sigma}\rt). 
    \end{align}
where the last inequality is from the definition of smooth calibration, since $\sigma w(p)$ is $1$-Lipschitz. Further we have
\begin{align}
    \label{eq:smooth2}
    \lvert \E[\ell(\p_\z,\y^\p) - \ell(\p_\z,\y^{\p_\z})]\rvert \lesssim \sigma
\end{align}
This is an application of \Cref{lemma:tv} to the bounded function $\ell(\p_\z, y)$, and then sampling $y$ according to $\y^\p$ and $\y^{\p_\z}$, and then taking expectations over $\p_\z$. 

Equation \eqref{eq:smooth0} follows from Equations \eqref{eq:smooth1} and \eqref{eq:smooth2} and the triangle inequality.
    \end{proof}

\eat{
\begin{theorem}\label{theorem:omniprediction-main-2}
    Let $\mu$ be the distribution of $(\p,\y)$.  For any proper loss $\ell \in \mL$ and any post-processing  $\kappa:[0,1]\to [0,1]$:
    \[
        \E[ \ell(\p_\z,\y)] - \E[\ell(\kappa(\q_\z),\y)] \lesssim \sigma + \frac{1}{\sigma}\Bigr( \smce(\mu) + \E\lvert\p-\q\rvert \Bigr) 
    \]
\end{theorem}
}
\begin{proof}[Proof of \Cref{theorem:omniprediction-intro}]
\Cref{thm:old-ldce-as-w}, implies the existence of a joint distribution $(\p,\q,\y)$ of random variables in $[0,1]\times[0,1]\times\{0,1\}$ such that $(\p, \y) \sim \mu$, $(\q, \y)\sim \nu$ and $\E[|\p - \q|] = W(\mu, \nu)$.

Let $\z \in [\pm \sigma]$ for $\sigma\in(0,1]$. Recall that $\y^{\p_\z} \in \zo$ is such that $\E[\y^{\p_\z}|\p_\z] = \p_\z$, in other words, $\p_\z$ is the Bayes optimal predictor for $\y^{\p_\z}$. Since $\ell$ is proper, for any $\kappa:[0,1]\to [0,1]$:
    \begin{align}
    \label{eq:bayes-opt} \E[\ell(\p_\z,\y^{\p_\z})] \le \E[\ell(\kappa(\p_\z),\y^{\p_\z})].
    \end{align}

    By applying Lemma \ref{lemma:main1} to the loss function $\ell$,    \begin{equation}
\left|\E[\ell(\p_\z,\y)] - \E[\ell(\p_\z,\y^{\p_\z})]\right| \lesssim \sigma + \frac{\smce(\mu)}{\sigma}.\label{equation:omni-main-0}
    \end{equation}
    Since Lemma \ref{lemma:main1} holds for any bounded loss, even if it is not proper,  we may apply it for $\tilde\ell(p,y)=\ell(\kappa(p),y)$ which is clearly bounded. We obtain the following:
    \begin{equation}
\left|\E[\ell(\kappa(\p_\z),\y^{\p_\z})] - \E[\ell(\kappa(\p_\z),\y)] \right| \lesssim \sigma + \frac{\smce(\mu)}{\sigma}. \label{equation:omni-main-1}
    \end{equation}
    By \Cref{lemma:smoothing}, for $y \in \zo$, the function 
    \[f_y(p) = \E_{\z \sim [\pm \sigma]}[\ell(\kappa(p_\z),y)]\] 
    is $O(1/\sigma)$-Lipschitz in $p$, hence $|f_y(p) -f_y(q)| \leq O(|p - q|/\sigma)$. 
    Taking expectations over $\p,\q$ and $\y$, we get:
    \begin{equation}
        \bigl\lvert\E[\ell(\kappa(\p_\z),\y)] - \E[\ell(\kappa(\q_\z),\y)]\bigr\rvert \lesssim \frac{ \E[\lvert\p-\q\rvert]}{\sigma}.\label{equation:omni-main-2}
    \end{equation}
    Using Equations 
    \eqref{equation:omni-main-1} and \eqref{equation:omni-main-2}, 
the triangle inequality, and the fact that $\E[|\p - \q|] = W(\mu, \nu)$
    \begin{equation}
    \label{eq:triangle-ineq} \bigl\lvert\E[\ell(\kappa(\p_\z),\y^{\p_\z})] - \E[\ell(\kappa(\q_\z),\y)]\bigr\rvert \lesssim \sigma + \frac{\smce(\mu)}{\sigma} + \frac{W(\mu, \nu)}{\sigma}. 
    \end{equation}
    The desired result follows by starting from Equation \eqref{eq:bayes-opt} and
    \begin{itemize}
        \item 
    Replacing $\E[\ell(\p_\z, \y^{\p_\z})]$ on the LHS with $\E[\ell(\p_\z, \y)]$ by  Equation \eqref{equation:omni-main-0}
    \item Replacing 
$\E[\ell(\kappa(\p_\z), \y^{\p_\z})]$ on the RHS with $\E[\ell(\kappa(\q_\z), \y)]$ by Equation \eqref{eq:triangle-ineq}.
    \end{itemize}
\end{proof}

\subsection{Omniprediction With Respect to Nearby Calibrated Predictors}\label{section:nearby-calibrated-omni}

In this section, we
prove \Cref{theorem:omniprediction-intro2}, where we restrict the baselines class to calibrated predictors in $\cl^{=\tau(\mD)}$. We restate the theorem below.

\OmniCal*

\eat{
In fact, we will show the following stronger statement.

\begin{theorem}\label{corollary:calibrated-omniprediction-general}
    Let $p \in \PX$, $\mD$ be a distribution on $\X \times \zo$, $\mu = p \circ \mD$ and $\nu \in \J^{=\tau(\mD)}$.
For any bounded proper loss $\ell \in \mL$,  and $\sigma \in (0,1]$, we have:
    \[
        \E_{\substack{(\p, \y) \sim \mu\\ \z \sim [\pm \sigma]}}[\ell(\p_\z,\y)] - \E_{(\q, \y) \sim \nu}[\ell(\q,\y)] \lesssim \sigma + \frac{\smce(\mu)}{\sigma} + W(\mu,\nu).
    \]
\end{theorem}
}
Our proof will use the $V$-shaped losses introduced in \cite{KLST23}. These are losses of the form  
\[ \ell_v(p,y) = -(y-v)\sign(p-v) , \ v\in[0,1].\] 
It is easy to see that $\ell_v \in \mL$. In fact, these functions form a basis for $\mL$: \cite{KLST23} show that every loss $\ell \in \mL$ can be written as
\begin{align}
\label{eq:klst}
 \ell(p,q) = \sum_v \lambda_v \ell_v(p,q), \ \  \lambda_v \geq 0, \sum_v \lambda_v \leq 2.
 \end{align}
For a precise statement that accounts for convergence issues, see \cite[Lemma 3.5]{GopalanSTT26}. This statement ignores linear terms of the form $ay + b$ for $a, b \in \R$. Since the contribution of such terms is independent of the prediction $p$, we can ignore them as long as we compare \preds which have the same marginal distribution over $\y$.

For a loss $\ell_v$ and a perfectly calibrated \pred $\mu \in \cl$, we can compute the expected loss as follows:
    \begin{align}
    \label{eq:v-cal}
    \E_{(\p, \y) \sim \mu}[\ell_v(\p,\y)] &= \E[-(\y- v)\sign(\p -v)] = \E[-(\p - v)\sign(\p -v)] = - \E[\lvert\p-v\rvert] 
    \end{align}

Next we prove some helpful lemmas. The first shows that when we restrict our attention to fully calibrated predictors, the smoothing operation does not significantly change the value of any proper loss. 

\begin{lemma}\label{lemma:calibration-makes-smoothing-irrelevant}
    Let $(\q^*,\y)\in\calib$ and let $\z \sim [\pm \sigma]$. Then  for any proper loss $\ell \in \mL$:
    \[
        \bigl\lvert\E[\ell(\q^*,\y)] - \E[\ell(\q^*_\z,\y)] \bigr\rvert \lesssim \sigma.
    \]
\end{lemma}
\begin{proof}
It suffices to show the desired result for $\ell_v$, for $v \in [0,1]$. Since $(\q^*, \y) \in \cl$, by Equation \eqref{eq:v-cal}, $\E[\ell_v(\q^*, \y)] = -\E[|\q -v|]$. 
    Similarly, we can write
    \begin{align*}
\E[\ell_v(\q^*_\z,\y)] &= \E[(\y- v)\sign(\q^*_\z -v)]  = -\E[(\q^*-v)\cdot\sign(\q^*_\z-v)].
    \end{align*}
    It follows that
    \[ \E[\ell_v(\q^*_\z,\y)] = -\E[(\q^*-v)\sign(\q^*_\z-v)] \geq -\E[|\q^* -v|] = \E[\ell_v(\q^*,\y)].\]
    For the other direction, observe that $\q^*-v = \q^*-\q^*_\z + \q^*_\z-v$. Therefore:
    \begin{align*}
\E[\ell_v(\q^*_\z,\y)] &=  -\E[(\q^*-v)\cdot\sign(\q^*_\z-v)]\\ 
        &=  - \E[\lvert\q^*_\z-v\rvert] -\E[(\q^*-\q^*_\z)\cdot\sign(\q^*_\z-v)] \\
        &\le - \E[\lvert\q^*_\z-v\rvert] + \E[\lvert\q^*-\q^*_\z\rvert] \\
        &\le - \E[\lvert\q^*-v\rvert] + 2\E[\lvert\q^*-\q^*_\z\rvert] \\
        &\le \E[\ell_v(\q^*,\y)] + 2 \E[\lvert\z\rvert]\\
        &\le \E[\ell_v(\q^*,\y)] + 2 \sigma.
    \end{align*}
\end{proof}

Our next lemma shows that when restricted to the space $\cl$ of calibrated \preds, every loss in $\mL$ is Lipschitz in the earth mover distance between them. 

\begin{lemma}\label{lemma:calibration-implies-lipschitz-losses}
     Let $\mu, \nu \in \cl^{=\tau}$. Then, for any $\ell \in \mL$,
    \[
        \bigl\lvert \E_{(\q, \y) \sim \mu} [\ell(\q,\y)] - \E_{(\q', \y') \sim \nu} [\ell(\q',\y')] \bigr\rvert \leq W(\mu,\nu).
    \]
\end{lemma}
\begin{proof}
    By the result of \cite{KLST23}, it is sufficient to show the claim for $\ell_v$. By Equation \eqref{eq:v-cal}, 
    \begin{align}
        \label{eq:ell-abs} \E[\ell_v(\q,\y)] = -\E\lvert\q-v\rvert, \ \  \E[\ell_v(\q', \y')] = -\E\lvert\q' -v\rvert
    \end{align}
    By the definition of earth mover's distance, there exists a coupling 
$(\q,\q')$ of $\q$ and $\q'$ such that $\E[\lvert\q-\q'\rvert] = W(\q, \q')$ (and this is minimum possible over all couplings). By the triangle inequality, 
\[ \left| \E[|\q -v|] - \E[|\q' -v|] \right| \leq \E[|\q - \q'|] = W(\q, \q'). \]
To finish the proof, we observe that
\[ W(\q, \q') \leq W((\q, \y), (\q', \y')) = W(\mu, \nu).\]
This holds true because a coupling of $(\q, \y)$ and $(\q', \y')$ gives a coupling of $\q|\y = b$ and $\q'|\y' =b$ for $b \in \zo$. Hence it induces a coupling of $\q, \q'$ of the same cost.   
\end{proof}

We are now ready to prove 
\Cref{theorem:omniprediction-intro2}.

\begin{proof}[Proof of \Cref{theorem:omniprediction-intro2}]
    Consider the joint distribution $(\p,\q^*,\y)$ such that $(\q^*,\y) \sim \nu^*$ for $\nu^* \in \calib^{=\tau}$, $(\p,\y) \sim \mu$ and
    $\E[\lvert\p-\q^*\rvert]$ is the minimum possible. By  \Cref{thm:triples} and Corollary \ref{cor:bghn-redux}, this minimum is $\lCE(\mu) = W(\mu, \nu^*)$.
    Note that
    the distribution $\nu^*$ of $(\q^*, \y)$ need not be the same as $\nu$, it could be that $W(\mu, \nu^*) \leq W(\mu, \nu)$.
    We invoke \Cref{theorem:omniprediction-intro} with $\kappa$ being the identity function to obtain the following:
    \begin{align}
        \label{eq:thm-id} \E[\ell(\p_\z,\y)] - \E[\ell(\q^*_\z,\y)] \lesssim \sigma + \frac{1}{\sigma}(\smce(\mu) + E[\lvert\p-\q^*\rvert]) \lesssim  \sigma + \frac{\smce(\mu)}{\sigma}.
    \end{align}
    By \Cref{lemma:calibration-makes-smoothing-irrelevant}, we get:
    \[
    \E[\ell(\q^*_\z,\y)] - \E[\ell(\q^*,\y)] \lesssim \sigma .
    \]
    By \Cref{lemma:calibration-implies-lipschitz-losses} applied to the calibrated \preds $\nu^*$ and $\nu$, 
    \begin{align*}
        |\E_{(\q^*, \y) \sim \nu^*}[\ell(\q^*,\y)]-\E_{(\q, \y) \sim \nu} [\ell(\q,\y)]| & \le W(\q^*,\q)\\ 
        & \leq 
        W(\nu^*, \nu)\\ 
        &\leq W(\mu, \nu^*) + W(\mu, \nu)\\ 
        &\leq 2\smCE(\mu) + W(\mu, \nu).
    \end{align*}
    We obtain the desired result by replacing $\ell(\q^*_\z, \y)$ on the LHS of Equation \eqref{eq:thm-id} with $\ell(\q, \y^\q)$, at a further cost of $O(\sigma + \smCE(\mu)) + W(\mu,\nu)$.
\end{proof}

\subsection{Tightness of Our Result}

We show that \Cref{theorem:omniprediction-intro} cannot be strengthened in any substantial way. 

\paragraph{Necessity of Smoothing} We begin by showing that if a predictor $\p$ is merely guaranteed to be smoothly calibrated, then one cannot hope to obtain an omniprediction guarantee even with respect to nearby post-processings $\q=\kappa(\p)$ of $\p$, unless both $\p$ and $\q$ are smoothed. Therefore, smoothing is necessary for the bound of \Cref{theorem:omniprediction-intro}.

\begin{theorem}\label{theorem:smoothing-necessity}
    For any $\varepsilon\in (0,1),\sigma\in[0,1]$, there are joint random variables $(\p,\q,\y) \in [0,1]\times[0,1]\times\{0,1\}$ and a proper loss $\ell:[0,1]\times\{0,1\}\to [-1,1]$ such that:
    \begin{itemize}
        \item $\q = \kappa(\p)$ for $\kappa(p) = 1-p$.
        \item $\E[\lvert\q-\p\rvert] \le 2\varepsilon$ and $\smce(\mu) \le  2\varepsilon$, where $(\p, \y) \sim \mu$.
        \item We have that $\E[\ell(\p,\y)] =1/2$, $ \E[\ell(\q,\y)] = - 1/2$ and, moreover:
        \[
            \E[\ell(\p_\z,\y)] \ge 0 \ge \E[\ell(\q_\z,\y)]
        \]
    \end{itemize}
    In particular, we have:
    \begin{align*}
        \E[\ell(\p,\y)] - \E[\ell(\q_\z,\y)] \ge \frac{1}{2} \;\text{ and }\;
        \E[\ell(\p_\z,\y)] - \E[\ell(\q,\y)] \ge \frac{1}{2}
    \end{align*}
\end{theorem}

\begin{proof}
    Consider a pair of random variables $(\x,\y)$ that takes the value $(x_0,0)$ with probability $1/2$ and the value $(x_1,1)$ with probability $1/2$. Let $\p,\q$ be as follows:
    \begin{align*}
        \p = \begin{cases}
            1/2+\varepsilon,\text{ if }\x=x_0 \\
            1/2-\varepsilon,\text{ if }\x=x_1
        \end{cases}, \qquad
        \q = \begin{cases}
            1/2-\varepsilon,\text{ if }\x=x_0 \\
            1/2+\varepsilon,\text{ if }\x=x_1
        \end{cases}
    \end{align*}
    Observe that $\q = 1-\p$ and that $\smce(\mu)\lesssim \varepsilon$, because the distribution $\nu$ of $(\p^*,\y)$ where $\p^* \equiv 1/2$ satisfies $\nu\in\calib$ and $W(\mu,\nu)\lesssim \varepsilon$ (see \Cref{thm:emc-smce} and \Cref{corollary:smooth-cdl}). Moreover, we have that $|\q-\p|=2\varepsilon$ with probability $1$ and therefore $\E\lvert\q-\p\rvert= 2\varepsilon$.

    The random variable $\p$ is always biased in the wrong direction and therefore $\sign(\q-1/2) = -\sign(\y-1/2)$ with probability $1$.
    On the other hand, the variable $\q$ is biased in the right direction and therefore $\sign(\q-1/2) = \sign(\y-1/2)$ with probability $1$. Recall that $\ell_{1/2}(p,y) = -\frac{1}{2}\cdot \sign(y-1/2)\cdot \sign(p-1/2)$, and $\ell_{1/2}$ is a proper loss. It follows that:
    \begin{align*}
        \E[\ell_{1/2}(\p,\y)] &= \frac{1}{2}\\
        \E[\ell_{1/2}(\q,\y)] &= -\frac{1}{2}
    \end{align*}

    The random variable $\p_\z$ achieves worse misclassification error than $\p^*$, because it is never biased toward predicting the correct label. In particular, even if $\sigma = 1$, the variable $\p_\z$ does not beat random guessing. Similarly, $\q_\z$ is always at least as good as random guessing. Overall, we have:
    \[
        \E[\ell_{1/2}(\p_\z,\y)] \ge \E[\ell_{1/2}(\p^*,\y)] = 0\ge \E[\ell_{1/2}(\q_\z,\y)],
    \]
    which concludes the proof of \Cref{theorem:smoothing-necessity}.
\end{proof}

\begin{remark}
    The predictors $\p,\q$ in \Cref{theorem:smoothing-necessity} correspond to the bad ($b$) and good ($g$) predictors over the two point space $\{x_0,x_1\}$ defined in \Cref{section:technical-overview-omniprediction}. Namely, if $\x\sim\{x_0,x_1\}$, then $\p = b(\x)$ and $\q = g(\x)$. For the loss $\ell$, we choose $\ell_{1/2}$, which is a V-shaped loss as described in \Cref{section:nearby-calibrated-omni} and we have $\ell_{1/2} = \ell_{\{0,1\}} - 1/2$, where $\ell_{\{0,1\}}$ is the proper version of the 0-1 loss, as defined in \Cref{section:technical-overview-omniprediction}.
\end{remark}

\paragraph{Omniprediction Error with respect to Nearby Predictors.} Using a similar construction as the one for \Cref{theorem:smoothing-necessity} we show the necessity of the term $W(\mu,\nu)/\sigma$ in \Cref{theorem:omniprediction-intro}, even if $\mu$ is calibrated and $\smCE(\mu) = 0$.

\begin{theorem}\label{theorem:lower-bound-1}
    Let $0<\varepsilon\le\sigma\le 1$. There exists \preds $\mu \in \cl^{=1/2}, \nu \in \J^{=1/2}$ such that $W(\mu, \nu) \leq \varepsilon$, and for $\z \sim [\pm \sigma]$, 
    \[
        \E_{\substack{(\p, \y) \sim \mu\\ \z \sim [\pm \sigma]}} [\ell_{1/2}(\p_\z,\y)] - \E_{\substack{(\q, \y) \sim \nu\\ \z \sim [\pm \sigma]}} [\ell_{1/2}(\q_\z,\y)] \geq  \frac{\varepsilon}{2\sigma}.
    \]
\end{theorem}

\begin{proof}
    Suppose that $\y \in \zo$ is uniformly random, that $\p = 1/2$ with probability $1$, and that 
    \[
        \q =
        \begin{cases}
            \frac{1}{2}+\varepsilon\ \text{ if }\y=1 \\
            \frac{1}{2}-\varepsilon\ \text{ if }\y=0
        \end{cases}
    \]
    By \Cref{thm:old-ldce-as-w}, $W(\mu, \nu) \leq\E[\lvert\p-\q\rvert] = \varepsilon$ and $\p$ is calibrated. Note that $\y$ is independent of  $\p$ and $\p_\z$. In particular, we have $\E[\y|\p_\z] = 1/2$, hence 
    \[
        \E[\ell_{1/2}(\p_\z,\y)] = -\E[(\y - 1/2)\sign(\p_\z - 1/2)] = 0.
    \]
    For $\q_\z$, we can write
    \begin{align*} 
    \E[\ell_{1/2}(\q_\z, \y)] &= -\E[(\y - 1/2)\sign(\q_\z - 1/2)].
    \end{align*}
    Observe that the loss is positive if and only if $\z\cdot \sign(\y-1/2) \le -\varepsilon$, which happens with probability $(\sigma-\varepsilon)/2\sigma$, and results in a loss of $1/2$, else the loss is $-1/2$. Overall, we have:
    \[
        \E[\ell(\q_\z,\y)] \leq \fr{2}\lt(\frac{\sigma-\varepsilon}{2\sigma} -\frac{\sigma+\varepsilon}{2\sigma}\rt) = -\frac{\varepsilon}{2\sigma}\,,
    \]
    which completes the proof.
\end{proof}

\paragraph{Tightness of Post-Processing Omniprediction Gap} Finally, we show that the bound of \Cref{corollary:smooth-cdl} is tight up to constant multiplicative factors.

\begin{theorem}\label{theorem:lower-bound-2}
    Let $0\le \varepsilon\le \sigma\le 1/12$. There is a PLD $\mu$ with $\smce(\mu) \lesssim \varepsilon$, a proper loss $\ell \in \mL$ and a post-processing $\kappa:[0,1]\to [0,1]$ such that if $(\p,\y)\sim\mu$ and $\z \sim [\pm \sigma]$:
    \[
        \E[\ell(\p_\z,\y)] - \E[\ell(\kappa(\p_\z),\y)] \gtrsim \sigma + \frac{\varepsilon}{\sigma}.
    \]
\end{theorem}

\begin{proof}
    Consider the mixture PLD $\mu = (1/2)\mu^{(0)} + (1/2)\mu^{(1)}$, where $\mu^{(1)}(1,1) = 1$ and $\mu^{(0)}$ is defined as follows for $\delta = \varepsilon/\sigma$:
    \begin{align*}
        \mu^{(0)}(1/2,y) &= \frac{1-\delta}{2}\text{ for }y\in\{0,1\},\qquad \mu^{(0)}\Bigr(\frac{1}{2}-2\sigma, 1\Bigr) = \mu^{(0)}\Bigr(\frac{1}{2}+2\sigma, 0\Bigr) = \frac{\delta}{2}
    \end{align*}
    Clearly $\mu^{(1)}\in \calib$. Further, we claim that $\smce(\mu^{(0)})  \leq 2 \delta \sigma = 2\varepsilon$, since the earth mover's distance between the distribution $\mu^{(0)}$ and the trivially calibrated distribution $(1/2, \mathsf{Be}(1/2))$ is $\Theta(\delta \sigma)$, which is attained by moving the mass of the points $(1/2-2\sigma,1), (1/2+2\sigma,0)$ to the points $(1/2,1), (1/2,0)$. Overall, we have that $\smce(\mu) \leq \varepsilon$.

    \paragraph{Proof Structure} For the following, we let $\E_{i}$ denote the expectation over $\z\sim [-\sigma,\sigma]$ and $(\p,\y)\sim \mu^{(i)}$. We will first show the following:
    \begin{itemize}
        \item There is a proper loss $\ell^{(0)}$ and a post-processing function $\kappa^{(0)}$ such that:
        \[
            \E_{\mu^{(0)}}[\ell^{(0)}(\p_\z,\y)] - \E_{\mu^{(0)}}[\ell^{(0)}(\kappa^{(0)}(\p_\z),\y)] \gtrsim \frac{\varepsilon}{\sigma},
        \]
        \item There is a proper loss $\ell^{(1)}$ and a post-processing function $\kappa^{(1)}$ such that:
        \[
            \E_{\mu^{(1)}}[\ell^{(1)}(\p_\z,\y)] - \E_{\mu^{(1)}}[\ell^{(1)}(\kappa^{(1)}(\p_\z),\y)] \gtrsim {\sigma}.
        \]
    \end{itemize}
    Then, we will combine these constructions to obtain a proper loss $\ell$ and post-processing function $\kappa$ with the desired properties with respect to $\mu$.

    \paragraph{The First Term} We let $\ell^{(0)}(p,y) = \ell_{1/2}$, and define $\kappa^{(0)}$ as follows:
    \[
        \kappa^{(0)}(p) = 
        \begin{cases}
            \frac{1}{2}-\sigma, \text{ if }p>\frac{1}{2}+\sigma \\
            \frac{1}{2}, \text{ if }p\in [\frac{1}{2}-\sigma, \frac{1}{2}+\sigma]\\
            \frac{1}{2}+\sigma, \text{ if }p < \frac{1}{2}-\sigma
        \end{cases}
    \]

    We wish to lower bound the quantity $\E_{\mu^{(0)}}[\ell^{(0)}(\p_\z,\y)] - \E_{\mu^{(0)}}[\ell^{(0)}(\kappa^{(0)}(\p_\z),\y)]$. Note that if $\p = 1/2$, then the terms corresponding to $\y=0$ and $\y=1$ cancel out by symmetry. Therefore, it suffices to account for the terms corresponding to $\p = 1/2\pm 2\sigma$, since the support of $\mu^{(0)}$ does not contain any other points. Note that conditioned on these cases, i.e., $|\p-1/2| = 2\sigma$, we have: 
    \[\sign(\p_z-1/2) = \sign(\p-1/2) \neq \sign(\y-1/2)=\sign(\kappa^{(0)}(\p_z)-1/2)\;\text{  for any } z\in (-\sigma,\sigma).\]
    Therefore, we obtain the following by observing that $\ell^{(0)}$ is essentially the binary misclassification error, and the probability that $|\p-1/2| = 2\sigma$ under $\mu^{(0)}$ is $\delta$:
    \begin{equation}
        \E_{\mu^{(0)}}[\ell^{(0)}(\p_\z,\y) - \ell^{(0)}(\kappa^{(0)}(\p_\z),\y)] \gtrsim \delta = \frac{\varepsilon}{\sigma}. \label{equation:term-1}
    \end{equation}

    \paragraph{The Second Term} We let 
    \[ \ell^{(1)}(p,y) = \ell_{1-\sigma/2}(p,y)=-(y-1+\frac{\sigma}{2})\cdot \sign(p-1+\frac{\sigma}{2}).\]
    We define $\kappa^{(1)}(p) = 1$ for all $p\in[0,1]$. Observe that $\ell^{(1)}(p,1) = -(\sigma/2) \sign(p-1+\sigma/2)$. Therefore, we have that $\ell^{(1)}(\kappa(\p),\y) = -\sigma/2$, while for $\p_\z$ we have:
    \[
        \E_{\mu^{(1)}}[\ell^{(1)}(\p_\z,\y)] = \frac{\sigma}{2} \cdot \Bigr( \frac{1}{4} - \frac{3}{4} \Bigr) = - \frac{\sigma}{4}\,,
    \]
    which follows from the fact that for $p=1$, $\sign(\p_\z-1+\sigma/2) = -1$ if and only if $\z<-\sigma/2$. Therefore, 
    \begin{equation}
        \E_{\mu^{(1)}}[\ell^{(1)}(\p_\z,\y) - \ell^{(1)}(\kappa(\p),\y)] \gtrsim \sigma.\label{equation:term-2}
    \end{equation}

    \paragraph{Putting Everything Together} We let: 
    \[\ell(p,y) = \frac{1}{2}\ell_{1/2}(p,y)+ \frac{1}{2}\ell_{1-\frac{\sigma}{2}}(p,y)\] Note that $\ell$ is proper, because it is a convex combination of proper losses. We define $\kappa:[0,1]\to [0,1]$ as follows:
    \[
        \kappa(p) = 
        \begin{cases}
            \kappa^{(1)}(p), \text{ if }p\ge \frac{3}{4} \\
            \kappa^{(0)}(p), \text{ otherwise}
        \end{cases}.
    \]
We establish two properties of $\kappa$:
    \begin{claim}\label{claim:tightness-claim-1}
        For any $z\in (-\sigma,\sigma)$ and any $p$ in the support of $\mu$ we have: $\kappa(p) = \kappa(p_z)$.
    \end{claim}
    The above claim is true because each point $p$ in the support of $\mu$ is at least $\sigma$-far from the boundaries of the regions where $\kappa$ takes constant values. For example, if $\p = \frac{1}{2}+2\sigma$, then $\p_z\in(1/2+\sigma, 3/4)$. Hence, $\kappa(\p_z) = \kappa(\p) = 1/2-\sigma$.

    \begin{claim}\label{claim:tightness-claim-2}
        Let $\ell^{(0)} = \ell_{1/2}$ (with $v^{(0)} = 1/2$) and $\ell^{(1)}=\ell_{1-\sigma/2}$ (with $v^{(1)} = 1- \sigma/2$). Then we have:
        \[
            \E_{\z\sim [-\sigma,\sigma]}\E_{(\p,\y)\sim\mu^{(1-i)}}\Bigr[\ell^{(i)}(\p_\z,\y) - \ell^{(i)}(\kappa(\p),\y)\Bigr] = 0\text{ for }i\in\{0,1\}
        \]
    \end{claim}
    The above claim is true because the support of $\mu^{(1-i)}$ is at least $\sigma$-far from $v^{(i)}$ and the function $\kappa$ does not move points that are near $v^{(i)}$ to points near $v^{(1-i)}$. For example, consider the case $i=1$. In this case, we have $\p\in[1/2-2\sigma,1/2+2\sigma]$ with probability $1$ over $\mu^{(i-1)} = \mu^{(0)}$. For any fixed $y$, the value of $\ell^{(1)}(p,y)$ is determined by the value of $\sign(p-1+\sigma/2)$. However, for any $p\in[1/2-2\sigma,1/2+2\sigma]$, and any $z\in(-\sigma,\sigma)$ we have $\sign(p-1+\sigma/2) = \sign(p_z-1+\sigma/2) = \sign(\kappa(p)-1+\sigma/2) = -1$. Therefore, $\ell^{(1)}(p_z,y)-\ell^{(1)}(\kappa(p),y) = 0$. The case $i=0$ follows from a similar case analysis.

    Recall that we wish to give a lower bound on the following quantity:
    \begin{align*}
        \E[\ell(\p_\z,\y)] - \E[\ell(\kappa(\p_\z),\y)] = \frac{1}{2} \E_{\mu^{(0)}}[\ell(\p_\z,\y)-\ell(\kappa(\p_\z),\y)] + \frac{1}{2} \E_{\mu^{(1)}}[\ell(\p_\z,\y) - \ell(\kappa(\p_\z),\y)]\,,
    \end{align*}
    where we have used the fact that $\mu$ is the balanced mixture of $\mu^{(0)}$ and $\mu^{(1)}$. By using \Cref{claim:tightness-claim-1} we further obtain the following:
    \begin{align*}
        \E[\ell(\p_\z,\y)] - \E[\ell(\kappa(\p_\z),\y)] = \frac{1}{2} \E_{\mu^{(0)}}[\ell(\p_\z,\y)-\ell(\kappa(\p),\y)] + \frac{1}{2} \E_{\mu^{(1)}}[\ell(\p_\z,\y) - \ell(\kappa(\p),\y)]\,,
    \end{align*}

    Finally, recall that $\ell = (\ell^{(0)} + \ell^{(1)})/2$. Due to \Cref{claim:tightness-claim-2} we may ignore the term $\ell^{(0)}$ under $\mu^{(1)}$ as well as the term $\ell^{(1)}$ under $\mu^{(0)}$. We obtain the following:
    \begin{align}
        \E[\ell(\p_\z,\y)] - \E[\ell(\kappa(\p_\z),\y)] \gtrsim \E_{\mu^{(0)}}[\ell^{(0)}(\p_\z,\y) - \ell^{(0)}(\kappa(\p),\y)] + \E_{\mu^{(1)}}[\ell^{(1)}(\p_\z,\y) - \ell^{(1)}(\kappa(\p),\y)]\label{equation:putting-all-together}
    \end{align}
    The desired result follows by combining the above bound \eqref{equation:putting-all-together} with the bounds \eqref{equation:term-1} and \eqref{equation:term-2}.
\end{proof}

\section{Smooth Calibration Error Approximates Lower Distance}
\label{sec:ldce-smce}

In this section, we prove \cref{thm:emc-smce,thm:ldce-emc}, which immediately imply \cref{thm:ldce-smce} when combined. To begin, we recall their respective statements:

\EmcSmce*

\LdceEmc*

Of the two proofs, the proof of \cref{thm:emc-smce} is significantly simpler:

\begin{proof}[Proof of \Cref{thm:emc-smce}]
    To prove the lower bound, we must show that $\smce(\mu) < 2\eps$ whenever $\demc(\mu) < \eps$. First, by the definition of $\demc$, we can construct jointly distributed pairs $(\p, \y)$ and $(\p', \y')$ such that the former is distributed according to $\mu$, the latter is perfectly calibrated, and the two are close in expectation:
    \[
        \E\big[d\bigl((\p, \y), (\p', \y')\bigr)\bigr] = \E\lvert \p-\p' \rvert + \E\lvert \y-\y'\rvert < \eps.
    \]
    Next, to bound $\smce(\mu)$, consider any $1$-Lipschitz function $w : [0, 1] \to [-1, 1]$. Then,
    \[
        \E\bigl[w(\p)(\y-\p)\bigr] = \underbrace{\E\bigl[w(\p)\bigl((\y-\p) - (\y'-\p')\bigr)\bigr]}_{<\eps} + \underbrace{\E\bigl[\bigl(w(\p) - w(\p')\bigr)(\y'-\p')\bigr]}_{<\eps} + \underbrace{\E\bigl[w(\p')(\y'-\p')\bigr]}_{=0}.
    \]
    In the above equation, the first two terms on the right are $< \eps$ by the $\eps$-closeness of $(\p, \y)$ and $(\p', \y')$ in expectation. The third term is $0$ by calibration of $(\p', \y')$, which means that $\E[\y'|\p']=\p'$.

    To prove the upper bound, we use the observation from \cite{GH-survey} that when the smooth calibration error is small, it is easy to construct a nearby perfectly calibrated predictor by altering the conditional distribution of the label. Specifically, if $(\p, \y) \sim \mu$, then define $\tilde{\y}|\p \sim \Ber(\p)$ and let $\nu$ be the distribution of $(\p, \tilde{\y})$. Then, $\nu \in \calib$ and for any function $f : \mathcal{S} \to [-1, +1]$, we have
    \[
        \bigl\lvert \E[f(\p, \y)] - \E[f(\p, \tilde{\y})]\bigr\rvert = \Bigl\lvert\E\bigl[(f(\p, 1) - f(\p, 0))(\y - \p)\bigr]\Bigr\rvert.
    \]
    By Kantorovich-Rubinstein duality, the supremum of the left side over all $1$-Lipschitz $f$ is equal to $W(\mu, \nu)$. The right side is at most $2\smce(\mu)$ since the Lipschitz constant of $f(\cdot, 1) - f(\cdot, 0)$ is at most twice that of $f$. We conclude that $\demc(\mu) \le 2\smce(\mu)$.
\end{proof}

Next, we prove \Cref{thm:ldce-emc}, which equates $\demc$ and $\lCE$ via a novel but simple exchange argument. In what follows, a \emph{transport plan} $\pi$ from $\mu$ to $\calib$ is simply a coupling of $\mu$ with some $\nu \in \calib$. When we say that a plan $\pi$ \emph{moves} mass $m$ from $A \subseteq \mathcal{S}$ to $B \subseteq \mathcal{S}$, we mean that the associated coupling assigns mass $m$ to the set $A \times B$. Similarly, the \emph{cost} of a transport plan $\pi$ is
\[
    \E_{\substack{\bigl((\p, \y), (\p', \y')\bigr) \sim \pi}} \Bigl[d\bigl((\p, \y), (\p', \y')\bigr)\Bigr].
\]

\begin{figure*}
    \centering
    \begin{subfigure}[t]{0.5\textwidth}
        \centering
        \includegraphics[height=1.5in]{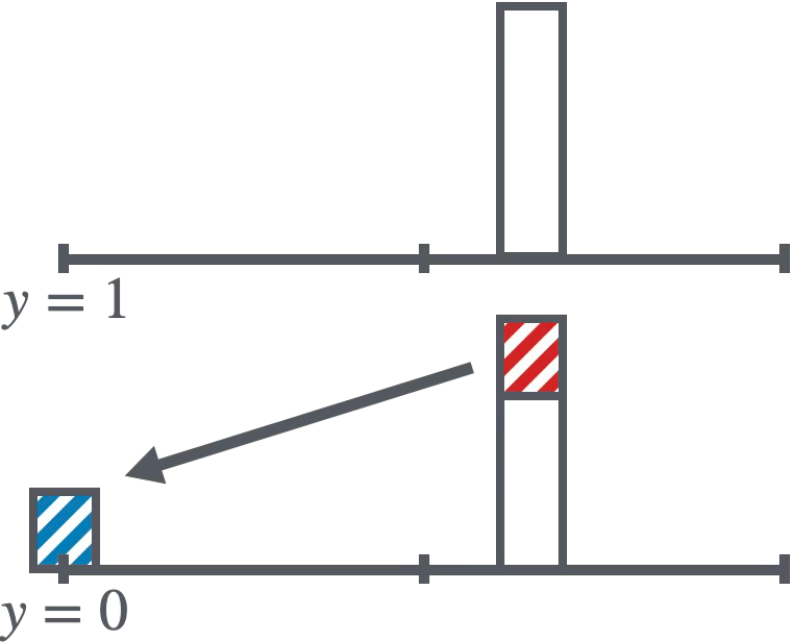}
        \caption{Move some mass to corner.}
        \label{fig:corner}
    \end{subfigure}
    ~ 
    \begin{subfigure}[t]{0.5\textwidth}
        \centering
        \includegraphics[height=1.5in]{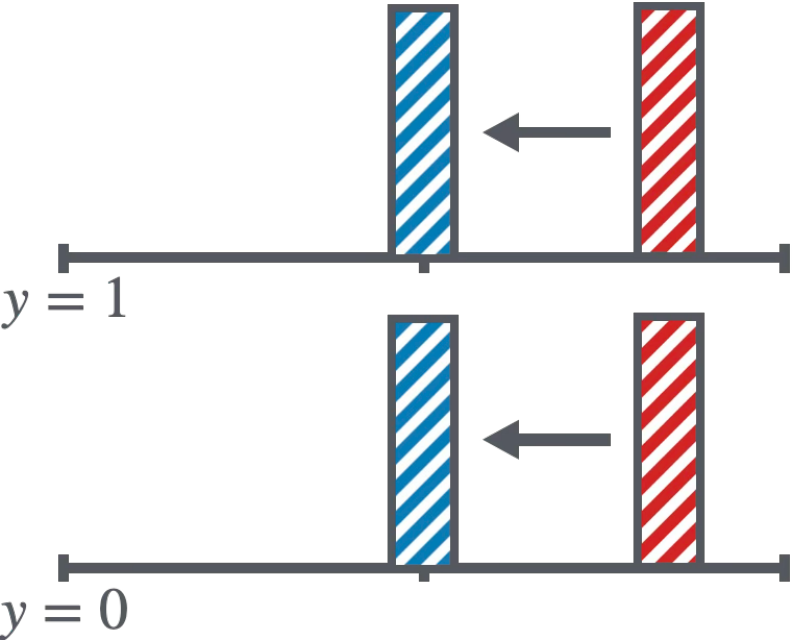}
        \caption{Move all mass where ratio dictates.}
        \label{fig:slide}
    \end{subfigure}
    \caption{Visualization of the proof of \Cref{thm:ldce-emc}.}
    \label{fig:within-segments}
    \label{fig:ldce-emc}
\end{figure*}

\begin{proof}[Proof of \Cref{thm:ldce-emc}]
    Let $S_0 = [0, 1] \times \{0\}$ and $S_1 = [0, 1] \times \{1\}$. As subsets of the domain $\mS$, these line segments correspond to the events $\y = 0$ and $\y = 1$.
    By definition, $\demc(\mu)$ is the infimum cost of all transport plans from $\mu$ to $\calib$, and $\lCE(\mu)$ is the infimum cost of transport plans from $\mu$ to $\calib$ that move no mass between $S_0$ and $S_1$. In more detail: We know from \cref{thm:triples} that $\lCE$ is the infimum cost over triples $(\p, \q, \y)$, as opposed to generic couplings of $(\p, \y)$ with $(\q, \y')$ for some $\y'$. This precisely means enforcing $\Pr[\y\neq\y'] = 0$ under the coupling, a.k.a. no mass is moved between $S_0$ and $S_1$. (Recall that just before the proof started, we clarified that ``moving mass from $A$ to $B$'' means the coupling placing mass on $A \times B$.)
    
    Thus, to prove that $\demc(\mu) = \lCE(\mu)$, it suffices to show that for every transport plan $\pi$ from $\mu$ to $\calib$, there exists a transport plan $\pi'$ from $\mu$ to $\calib$ that moves no mass between $S_0$ and $S_1$ and is no costlier than $\pi$. By taking limits, it suffices to prove the claim in the case that all distributions under consideration are discrete.
    
    To prove the claim, we first observe that $\pi$ can be viewed as the composition of two consecutive plans $\pi_1$ and $\pi_2$ such that $\pi_1$ only moves mass within each segment and $\pi_2$ only moves mass between $(p, 0)$ and $(p, 1)$ for various values $p \in [0, 1]$. Visually, $\pi_1$ moves mass ``horizontally,'' and $\pi_2$ moves mass ``vertically,'' as depicted in Figure~\ref{fig:within-segments}.
    
    Next, suppose that $\pi_2$ moves some mass from $(p, 0)$ to $(p, 1)$ for some value $p \in [0, 1]$. If $m_0$ and $m_1$ were the masses at these points just before the move, then $\pi_2$ must have moved precisely $c = pm_0 - (1-p)m_1$ mass between these two points to achieve calibration. Indeed, after such a move, which costs precisely $c$, the two points would have masses $m_0 - c = (1-p)(m_0 + m_1)$ and $m_1 + c = p(m_0 + m_1)$, respectively, which are in the correct ratio for calibration.
    
    Alternatively, as shown in Figure~\ref{fig:corner}, we could have achieved calibration by moving $m' = m_0 - (1/p-1)m_1$ mass from $(p, 0)$ to $(0, 0)$, which does not cross segments and has the same total cost of $pm' = c$. Indeed, after such a move, the two points would have masses $(1/p-1)m_1$ and $m_1$, respectively, which are in the correct ratio for calibration. (Yet another option, shown in Figure~\ref{fig:slide}, would be to move all the mass at $(p, 0)$ and $(p, 1)$ to the points $(p', 0)$ and $(p', 1)$, respectively, where $p' = m_1/(m_0 + m_1)$. This also has total cost $(m_0 + m_1)(p-p') = c$ and satisfies calibration.)
    
    At this point, we have shown that any plan $\pi$ can be transformed into a plan $\pi'$ with the same cost that moves no mass from $S_0$ to $S_1$. A similar argument applies to shipments from $S_1$ to $S_0$. Thus, there exists an optimal transport plan that moves no mass between segments in \emph{either} direction, as claimed.
\end{proof}

Next, we prove \cref{corollary:demc-equiv}. We recall the statement here for convenience.

\DemcEquiv*

\begin{proof}
    The inequalities $a \le b$, $a \le c$, and $b \le \underline{\mathrm{dCE}}(\mu)$ are all clear, since each inequality compares the infimum over a set to the infimum over a subset. \Cref{thm:ldce-emc} states that $\underline{\mathrm{dCE}}(\mu) = a$. Finally, our proof of the upper bound in \Cref{thm:emc-smce} implies that $c \le 2a$ since the only calibrated prediction-label distribution which preserves the marginal of $\p$ is that of $(\p, \tilde{\y})$, where $\tilde{\y}|\p \sim \mathrm{Bernoulli}(\p)$.
\end{proof}

Finally, we prove \cref{thm:old-ldce-as-w}. The proof is similar in spirit to the proof of \cref{thm:ldce-emc} but simpler. Again, we recall the relevant statement for convenience.

\OldLdceAsW*

\begin{proof}
    As in the proof of \cref{thm:ldce-emc}, consider two \preds $\mu$ and $\nu$. In this case, however, we do not assume that either $\mu$ or $\nu$ is calibrated, but instead require that $\Pr_{(\p,\y)\sim\mu}[\y=1]=\Pr_{(\p,\y)\sim\nu}[\y=1] = \tau$. Phrased differently, $\mu,\nu\in\J^{=\tau}$ for the same value $\tau \in [0, 1]$.
    
    Recall that $W(\mu,\nu)$ is, by definition, the infimum cost among all transport plans  $\pi$ from $\mu$ to $\nu$ (i.e. couplings of $\mu$ and $\nu$), which may or may not move mass between the segments $S_0$ and $S_1$. Observe that in the case that $\pi$ \emph{does} move mass from $S_0$ to $S_1$, it must also move the same amount of mass from $S_1$ back to $S_0$, from our assumption that both $\mu$ and $\nu$ place exactly $\tau$ mass on $S_1$. Similarly, if $\pi$ moves mass from $S_1$ to $S_0$, it must move the same amount of mass from $S_0$ back to $S_1$. 
    
    Ultimately, what we want to prove is the following claim: $W(\mu, \nu)$ is in fact equal to the infimum cost among the \emph{restricted} set of couplings of $(\p, \y) \sim \mu$ with $(\q, \y')\sim\nu$ such that $\Pr[\y=\y'] = 1$, a.k.a. triples $(\p, \q, \y)$ with $(\p,\y)\sim\mu$ and $(\q,\y)\sim\nu$. Viewing the coupling as a transport plan, this is equivalent to the restriction that the plan $\pi$ moves \emph{no} mass from $S_0$ to $S_1$ and \emph{no} mass from $S_1$ to $S_0$.

    To prove the claim, let $\pi$ be any transport plan from $\mu,\nu\in\J^{=\tau}$, and factor $\pi$ into two consecutive plans $\pi_1$ and $\pi_2$ as in the proof of \cref{thm:ldce-emc} and as depicted in Figure~\ref{fig:ldce-emc}, where $\pi_1$ only moves mass within segments (``horizontally''), and $\pi_2$ only moves mass between pairs of points of the form $(p, 0)$ and $(p, 1)$ (``vertically'').
    
    Suppose for the sake of contradiction that $\pi_2$ moved some positive mass $c > 0$ from $S_0$ to $S_1$ or vice versa. By the preceding discussion, we know that $\pi_2$ must in fact move $c$ mass from \emph{both} $S_0$ to $S_1$ and from $S_1$ to $S_0$. Recall that the metric under consideration is $\ell_1$:
    \[
        d\bigl((p, y), (p', y')\bigr) = \lvert p-p'\rvert + \lvert y-y'\rvert.
    \]
    Therefore, the cost of $\pi_2$ is $2c$. Had we instead exchanged these two shipments and moved them to their final destinations within segments, the above equality shows that we would have incurred a cost of $\le 2c$, proving the claim.
\end{proof}

\section{Inapproximability of Upper Distance}
\label{sec:udce-hard}

In this section, we state and prove the full version of \cref{thm:udce-hard-intro}.

\begin{theorem}
\label{thm:udce-hard}
    Fix $\eps,\eps' > 0$, $k \in \mathbb{N}$, and $\delta_{0j}, \delta_{1j} \in\R$ for $j \in [k]$. Consider the following tuples $(\X, \D_\X, p^*, p)$:
\begin{enumerate}[label=(\alph*)]
    \item Let $\X_a = \{L, R\}$ and define $\D_{\X_a}$, $p_a^*$, and $p_a$ as follows:
    \begin{center}
        \begin{tabular}{||c|c|c|c||} 
        \hline
             $x$ & $\D_{\X_a}(x)$ & $p_a^*(x)$ & $p_a(x)$ \\ [0.5ex] 
             \hline\hline
             $L$ & $1/2$ & $1/2$ & $1/2 - \eps$  \\ [0.5ex]
             \hline
             $R$ & $1/2$ & $1/2$ & $1/2+\eps$\\ [0.5ex]
             \hline
        \end{tabular}
    \end{center}
    \item Let $\X_b = \{L_0, L_1, R_0, R_1\}$ and define $\D_{\X_b}$, $p_b^*$, and $p_b$ as follows:
    \begin{center}
        \begin{tabular}{||c|c|c|c||} 
        \hline
             $x$ & $\D_{\X_b}(x)$ & $p_b^*(x)$ & $p_b(x)$ \\ [0.5ex] 
             \hline\hline
             $L_0$ & $\eps'$ & $1$ & $1/2 - \eps$  \\ [0.5ex]
             \hline
             $L_1$ & $1/2-\eps'$ & $1/2-\eps$ & $1/2-\eps$\\ [0.5ex]
             \hline
             $R_0$ & $\eps'$ & $0$ & $1/2+\eps$\\ [0.5ex]
             \hline
             $R_1$ & $1/2-\eps'$ & $1/2+\eps$ & $1/2+\eps$\\ [0.5ex]
             \hline
        \end{tabular}
    \end{center}
    \item Building on (a), let $\X_c = \{L, R\} \times [k/2]$ and define $\D_{\X_c}$, $p_c^*$, and $p_c$ as follows:
    \begin{center}
        \begin{tabular}{||c|c|c|c||} 
        \hline
             $x$ & $\D_{\X_c}(x)$ & $p_c^*(x)$ & $p_c(x)$ \\ [0.5ex] 
             \hline\hline
             $(L, j)$ & $1/k$ & $1/2$ & $1/2 - \eps + \delta_{0j}$  \\ [0.5ex]
             \hline
             $(R, j)$ & $1/k$ & $1/2$ & $1/2+\eps + \delta_{1j}$\\ [0.5ex]
             \hline
        \end{tabular}
    \end{center}
    \item Building on (b), let \[\X_d = \bigl(\{L_0, R_0\} \times \{1, \ldots, \eps' k\}\bigr) \cup \bigl(\{L_1, R_1\} \times \{\eps'k + 1, \ldots, k/2\}\bigr)\] and define $\D_{\X_d}$, $p_d^*$, and $p_d$ as follows:
    \begin{center}
        \begin{tabular}{||c|c|c|c||} 
        \hline
             $x$ & $\D_{\X_d}(x)$ & $p_d^*(x)$ & $p_d(x)$ \\ [0.5ex] 
             \hline\hline
             $(L_0,j)$ & $1/k$ & $1$ & $1/2 - \eps + \delta_{0j}$  \\ [0.5ex]
             \hline
             $(L_1,j)$ & $1/k$ & $1/2-\eps$ & $1/2-\eps + \delta_{0j}$\\ [0.5ex]
             \hline
             $(R_0,j)$ & $1/k$ & $0$ & $1/2+\eps + \delta_{1j}$\\ [0.5ex]
             \hline
             $(R_1,j)$ & $1/k$ & $1/2+\eps$ & $1/2+\eps + \delta_{1j}$\\ [0.5ex]
             \hline
        \end{tabular}
    \end{center}
\end{enumerate}
Suppose that $\eps' = \eps/(1+2\eps)$, that $\eps' k$ and $k/2$ are both integers, and that the parameters $\delta_{ij}$ are sampled i.i.d. from a continuous distribution over $[-\eps/2, \eps/2]$. Then,
\begin{enumerate}[label=(\roman*)]
    \item Case~(a), which has $\dCE \ge \Omega(\eps)$, is impossible to distinguish with any positive advantage from case~(b), which has $\dCE \le O(\eps^2)$, given any finite number of prediction-label samples.
    \item It requires at least $\Omega(\sqrt{k})$ prediction-label samples to distinguish (with constant advantage in expectation over the choice of $\delta_{ij}$) case~(c), which has $\uCE \ge \Omega(\eps)$, from case~(d), which has $\uCE \le O(\eps^2)$ (with probability $1$ over the choice of $\delta_{ij}$).
\end{enumerate}

In particular, $\dCE$ cannot be estimated within a better-than-quadratic factor from prediction-label samples, and $\uCE$ cannot be estimated within a better-than-quadratic factor from any number of prediction-label samples that is independent of the support size of the distribution of predictions.
\end{theorem}

\begin{figure}[t]
     \centering
     \begin{subfigure}[b]{\textwidth}
         \centering
         \includegraphics[height=1.5in]{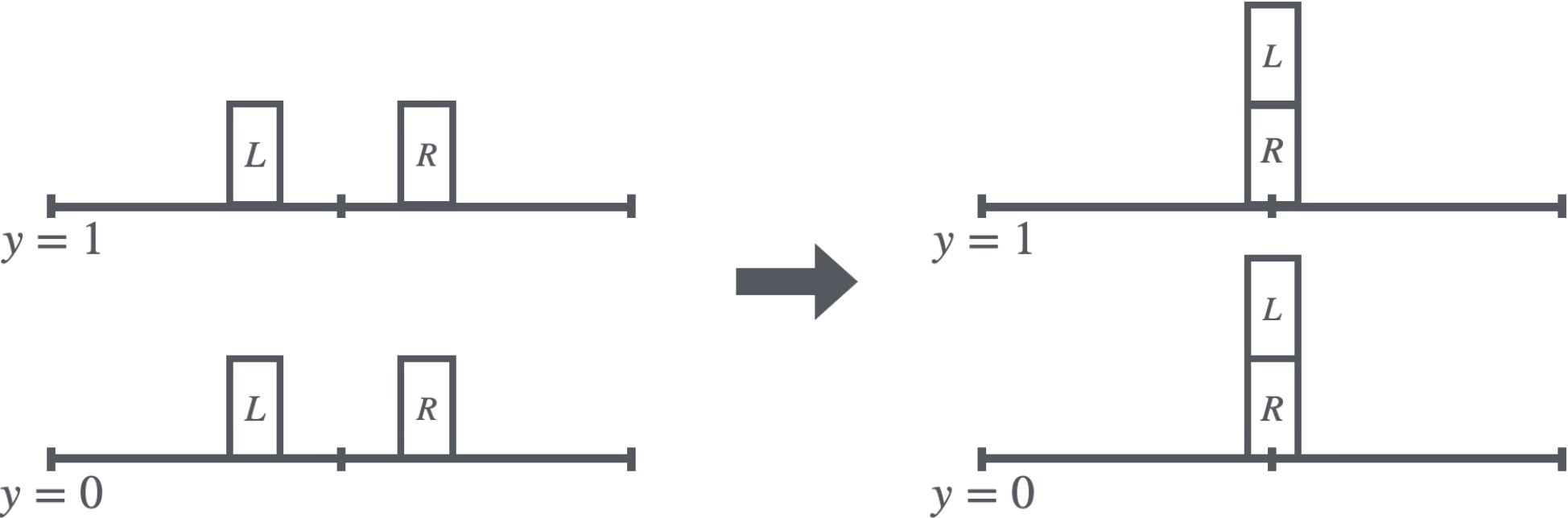}
         \caption{Cases (a) and (c) of \cref{thm:udce-hard}}
         \label{fig:lower-bounds-ac}
     \end{subfigure}
     \;
     \vspace{0.125in}
     \;
     \begin{subfigure}[b]{\textwidth}
         \centering
         \includegraphics[height=1.5in]{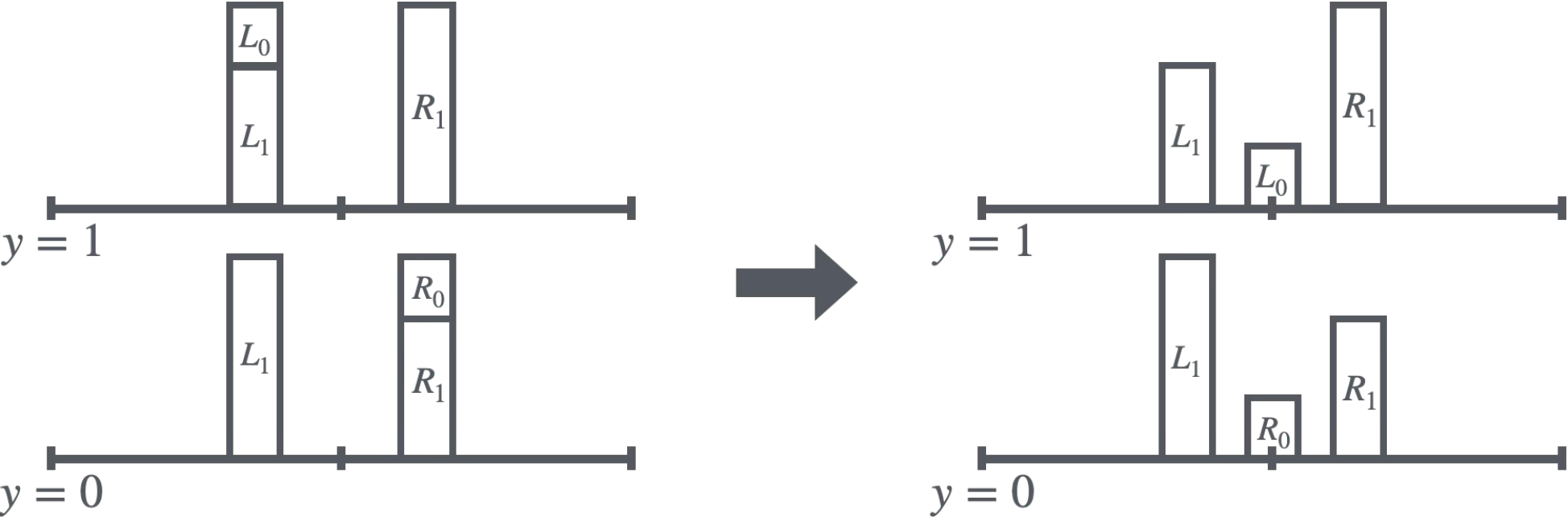}
         \caption{Cases (b) and (d) of \cref{thm:udce-hard}}
         \label{fig:lower-bounds-bd}
     \end{subfigure}
    \caption{Optimal transport to calibration for each case of \cref{thm:udce-hard}. In cases (a) and (b), the predictions are concentrated entirely on the points $1/2\pm \eps$, but in cases (c) and (d), they are instead scattered nearby.}
    \label{fig:lower-bounds}
\end{figure}

Observe that part (i) of \cref{thm:udce-hard} is precisely the prior result of \cite{BlasiokGHN23} regarding the inapproximability of $\dCE$ in the prediction-only access model. Part (ii) is our new result regarding $\uCE$. The proofs of both parts are most easily illustrated by Figure~\ref{fig:lower-bounds}.

\begin{proof}[Proof of \Cref{thm:udce-hard}]\;
\begin{enumerate}[label=(\roman*)]
    \item To see that case (a) has $\dCE \ge \Omega(\eps)$, observe that the Bayes optimal predictor $p_a^*$ is the constant $1/2$ function. Therefore, the only calibrated predictor on the domain $\X_a$ is the constant $1/2$ function itself. Since $p_a$ always outputs values exactly $\eps$-far away from $1/2$, it follows that $\dCE = \eps$ in case (a). In Figure~\ref{fig:lower-bounds-ac}, this corresponds to the movement of all mass (total of $1$) across a distance of length $\eps$, for a total cost of $1 \cdot \eps \le O(\eps)$.
    
    In contrast, the fact that case (b) has $\dCE \le O(\eps^2)$ is witnessed by calibrated predictor
    \[
        p_b'(x) = \begin{cases}
            1/2 &\text{if }x \in \{L_0, R_0\}\\
            1/2-\eps&\text{if }x=L_1,\\
            1/2+\eps&\text{if }x=R_1.
        \end{cases}
    \]
    Visually, $p_b'$ is obtained from $p_b$ by ``moving'' the two $\eps'$ masses at $x \in \{L_0, R_0\}$, each over a distance of length $\eps$, to the point $1/2$, as shown in Figure~\ref{fig:lower-bounds-bd}, for a total cost of $2\eps' \cdot \eps \le O(\eps^2)$. More formally,
    \[
        \E_{\x \sim \mD_{\X_b}} \bigl\lvert p_b(\x) - p'_b(\x) \bigr\rvert = 2\eps' \cdot \eps + (1-2\eps') \cdot 0 = \frac{2\eps^2}{1+2\eps} \le O(\eps^2).
    \]
    
    At this point, all that remains is to show that cases (a) and (b) give rise to identical \preds. This will imply that they cannot be distinguished from any finite number of prediction-label samples, as claimed. For this, we simply observe that in case (b), the expected label $\y$ given $p_b(\x) = 1/2-\eps$ is
    \[
        \E\Bigl[\y \,\Big|\,p_b(\x)=\frac{1}{2}-\eps\Bigr] = \frac{\eps' \cdot 1 + \Bigl(\frac{1}{2}-\eps'\Bigr)\cdot\Bigl(\frac{1}{2}-\eps\Bigr)}{\eps' + \Bigl(\frac{1}{2}-\eps'\Bigr)} = \frac{1}{2}.
    \]
    Similarly, one can check that $\E[\y|p_b(\x)=1/2+\eps] = 1/2$. Therefore, cases (a) and (b) give rise to the same \preds, as claimed, completing the proof of part (i).

    \item First, note that cases (c) and (d) are generalizations of cases (a) and (b), respectively, where instead of predicting $1/2\pm \eps$, we make predictions that are small random perturbations of $1/2\pm \eps$. Thus, our proof strategy for part (ii) will be similar in spirit to our strategy for part (i).

    To see that case (c) has $\uCE \ge \Omega(\eps)$, observe that the Bayes optimal predictor $p_c^*$ is the constant $1/2$ function. Therefore, the only calibrated predictor on the domain $\X_c$ is the constant $1/2$ function itself. Since $\lvert \delta_{ij} \rvert \le \eps/2$, the predictor $p_c$ always outputs values at least $\eps/2$-far away from $1/2$. It follows that $\uCE \ge \dCE \ge \eps/2$ in case (c) (with probability $1$ over the choice of $\delta_{ij}$). We also have $\uCE \le 3\eps/2$, and this is again depicted in Figure~\ref{fig:lower-bounds-ac}.

    The situation in case (d) is somewhat subtler. Intuitively, in case (d), there always \emph{exists} a ``cheap'' post-processing $\kappa$ certifying $\uCE \le O(\eps^2)$, but such a $\kappa$ is hard to actually \emph{find} from a handful of prediction-label samples. This is in contrast to parts (a) and (c), where such a post-processing does \emph{not} exist, and part (b), where such a post-processing exists \emph{and} is easy to construct from samples.

    More formally, the fact that case (d) has $\uCE \le O(\eps^2)$ is witnessed by the following post-processing function $\kappa_d$. The following definition of $\kappa_d$ mimics the transition from the predictor $p_b$ to $p_b'$ that we used in case (b) and visualized in Figure~\ref{fig:lower-bounds-bd}. In defining $\kappa_d$, we crucially use the fact that all perturbation terms $\delta_{ij}$ are distinct. This occurs with probability $1$ since they are drawn i.i.d. from a continuous distribution:
    \[
        \kappa_d(p) = \begin{cases}
            1/2&\text{if }p = 1/2-\eps + \delta_{0j}\text{ for some }j \le \eps'k,\\
            1/2-\eps &\text{if }p = 1/2-\eps+\delta_{0j} \text{ for some }j > \eps'k,\\
            1/2&\text{if }p = 1/2+\eps + \delta_{1j}\text{ for some }j \le \eps'k,\\
            1/2 + \eps &\text{if }p = 1/2+\eps + \delta_{1j}\text{ for some }j > \eps'k.
        \end{cases}
    \]
    Indeed, essentially the same analysis as in part (b) shows that $\kappa \circ p_d$ is calibrated and
    \[
        \E_{\x \sim \mD_{\X_d}} \bigl\lvert \kappa(p_d(\x)) - p_d(\x))\bigr\rvert \le O(\eps^2).
    \]
    Finally, we argue that cases (c) and (d) are hard to distinguish from few i.i.d. prediction-label samples $(p(\x_1), \y_1), \ldots, (p(\x_s), \y_s)$. For this, observe that in a sample of size $s$, the probability that we see a collision (i.e. sample $(p(\x), \y)$ for the same point $\x = x \in \X$ twice) is at most $\binom{s}{2}/k \le O(s^2/k)$. Moreover, by construction, conditional on seeing no collisions, each conditional distribution $\y_i|p(\x_i)$ is uniform on $\{0, 1\}$ in both cases (c) and (d) over the randomness in the choice of the parameters $\delta_{ij}$. This is clear for case (c), whereas for case (d), it follows from the same calculation as in part (b):
    \[
        \frac{\eps' \cdot 1 + \Bigl(\frac{1}{2}-\eps'\Bigr)\cdot\Bigl(\frac{1}{2}-\eps\Bigr)}{\eps' + \Bigl(\frac{1}{2}-\eps'\Bigr)} = \frac{1}{2}.
    \]
    We conclude that to distinguish cases (c) and (d) with at least constant advantage over the randomness in $\delta_{ij}$, we require at least $s \ge \Omega(\sqrt{k})$ prediction-label samples. \qedhere
\end{enumerate}
\end{proof}

\bibliography{refs}
\bibliographystyle{alpha}
\end{document}